\pdfoutput=1
\documentclass[10pt]{article}
\usepackage{enumerate}
\usepackage[OT1]{fontenc}
\usepackage{amsmath,amssymb}
\usepackage{natbib}
\usepackage[usenames,dvipsnames]{xcolor}
\usepackage{geometry}
\usepackage{dsfont}
\usepackage{pgfplots}
\usepackage{smile_new}
\usepackage{multirow}
\usepackage{rotating}
\usepackage{enumerate}
\usepackage{esvect}
\usepackage{tikz}
\usetikzlibrary{patterns}
\usetikzlibrary{arrows}
\usepackage{hyperref}
\definecolor{mydarkblue}{rgb}{0,0.08,0.45}
\hypersetup{
	colorlinks=true,
	linkcolor=mydarkblue,
	urlcolor=magenta,
	citecolor=mydarkblue,
}
\usepackage{algorithm}
\usepackage{algorithmic}
\usepackage[most,skins]{tcolorbox}
\definecolor{rliableolive}{HTML}{BBCC33}
\definecolor{rliableblue}{HTML}{77AADD}
\definecolor{rliablered}{HTML}{EE8866}
\definecolor{LightCyan}{rgb}{0.88,1,1}
\definecolor{darkblue}{HTML}{2878D9}
\definecolor{navyblue}{HTML}{0000FF}

\definecolor{keycolor}{HTML}{B5605C}         
\definecolor{morandibluelight}{HTML}{E2ECF7} 
\definecolor{morandibluedeep}{HTML}{6E8AAD}  

\definecolor{figureblue}{HTML}{2870A8}
\definecolor{figureorange}{HTML}{E26B38}
\definecolor{figuregreen}{HTML}{426043}
\definecolor{figurered}{HTML}{5F0014}
\definecolor{figurelightblue}{HTML}{76B5C3}
\newcommand{\kw}[1]{\textcolor{keycolor}{#1}}

\newtcolorbox{highlightbox}{
  enhanced,
  colback=morandibluelight,
  colframe=morandibluelight,
  boxrule=0pt, arc=0pt,
  left=10pt, right=8pt, top=6pt, bottom=6pt,
  borderline west={2pt}{0pt}{morandibluedeep}
}
\usepackage{enumitem}
\usepackage{acronym}
\acrodef{svd}[SVD]{Singular-Value Decomposition}
\acrodef{rgi}[RGI]{Relative Generalization Invariance}
\acrodef{llm}[LLM]{Large Language Model}
\acrodef{ntk}[NTK]{Neural Tangent Kernel}
\acrodef{mf}[MF]{Mean-Field}
\acrodef{ccc}[CCC]{Concordance Correlation Coefficient}
\acrodef{pcfg}[PCFG]{Probabilistic Context-Free Grammar}
\acrodef{nds}[NDS]{Normalized Directional Sharpness}
\newcommand{\gd}{\mathsf{GD}}
\newcommand{\muon}{\mathsf{Muon}}
\newcommand{\spec}{\texttt{spec}}
\newcommand{\sgn}{\texttt{sgn}}
\newcommand{\mat}{\texttt{mat}}

\newcommand{\Sc}{{\mathcal{S}}}

\newcommand{\opt}{\mathrm{opt}}
\DeclareMathOperator{\Diag}{Diag}
\usepackage{xspace}
\newcommand{\adam}{\mathsf{Adam}\xspace}

\theoremstyle{plain}

\usepackage{mathrsfs}
\usepackage{fullpage}

\usepackage[protrusion=false,expansion=true]{microtype}

\title{\huge Why Muon Outperforms Adam: A Curvature Perspective}
\author{Shuche Wang\textsuperscript{1,$\ast$}\quad Fengzhuo Zhang\textsuperscript{2,$\ast$,$\dagger$}\quad Jiaxiang Li\textsuperscript{3}\quad Dirk Bergemann\textsuperscript{2}\quad Zhuoran Yang\textsuperscript{2}\\
\textsuperscript{1}National University of Singapore\quad\textsuperscript{2}Yale University\quad \textsuperscript{3}University of Minnesota 
}
\date{May 6, 2026}
\begin{document}
\maketitle
\renewcommand\thefootnote{}\footnotetext{$\ast$ Equal contribution.}\footnotetext{$\dagger$ Project Lead.}\footnotetext{Email: Corresponding to \texttt{fzzhang@u.nus.edu}, \texttt{zhuoran.yang@yale.edu} }

\begin{abstract}
Muon improves training efficiency over Adam in large language-model training by about two times, but the local geometric source of this advantage remains unclear. Our work takes a first step toward demystifying Muon's superiority over Adam from a curvature perspective. First, we apply a second-order Taylor approximation to the training landscape and show that Muon achieves a larger one-step loss decrease than Adam at matched validation loss. 
The two optimizers have comparable first-order gains, but Muon consistently incurs a smaller second-order curvature penalty. 
Second, we decompose this curvature penalty into the squared update norm and Normalized Directional Sharpness (NDS). 
We find that Muon and Adam have comparable update norms, so Muon's smaller curvature penalty is driven by lower NDS, not update scale. Third, we study how training data and model structure shape Muon's NDS advantage. Using Zipf-Probabilistic Context-Free Grammar (PCFG) data with controlled imbalance, we show that data imbalance amplifies Muon's NDS advantage over Adam. A within-/cross-layer decomposition further shows that, in the middle and late stages of training, Muon's lower NDS is mainly sustained by smaller within-layer curvature. 
Beyond empirical evidence, we analyze stylized quadratic problems with heterogeneous curvature and gradient alignment toward high-curvature modes. 
We prove that Muon attains a smaller average NDS than GD by balancing update energy across curvature groups; when curvature heterogeneity is sufficiently strong, this also yields lower local quadratic loss after the same number of steps.
\end{abstract}
\section{Introduction}
Muon has emerged as a powerful alternative to Adam for \ac{llm} pretraining~\citep{jordan6muon}. 
It exploits the matrix structure of parameters by spectrally normalizing the gradient momentum matrix, effectively setting its nonzero singular values to the same scale. 
This matrix-aware design enables Muon to achieve up to about $2\times$ faster training than Adam in \ac{llm} pretraining across various model scales~\citep{liu2025muon,modded_nanogpt_2024,shah2025practical}.
Recent works have sought to explain Muon's advantage over Adam through the lenses of associative memory and data long-tailedness~\citep{wang2025muon,vasudeva2025muon}.

Different from existing mechanistic interpretations of Muon's superiority, we take a first step toward understanding Muon's advantage over Adam from the perspective of the optimization landscape. Specifically, we ask the following questions:
\begin{highlightbox}
\begin{itemize}[leftmargin=2em]
    \item[1.] \emph{What landscape property supports Muon's advantage over Adam?}
    \item[2.] \emph{How do pretraining factors, such as the training data and model structure, influence this property?}
\end{itemize}
\end{highlightbox}
To answer the first question, we analyze the one-step loss decrease of Muon and Adam via a second-order Taylor expansion of the training loss. We find that this expansion accurately predicts the realized loss decrease and reveals that, while the two optimizers achieve comparable first-order (gradient alignment) gains, Muon incurs a much smaller \kw{curvature penalty}. 
That is, Muon's larger realized loss decrease is primarily driven by its smaller second-order curvature cost. 
We further decompose the curvature penalty into contributions from the update norm and \ac{nds}, showing that Muon's smaller curvature penalty originates from \kw{lower \ac{nds}}, which is determined by the update direction, rather than from smaller update norms. 
Thus, our analysis identifies lower \ac{nds} as the main contributor to Muon's superiority over Adam.\looseness=-1

To answer the second question, we investigate how training data and model structure contribute to Muon's smaller \ac{nds}. 
Motivated by \citet{wang2025muon,vasudeva2025muon}, 
we first examine the role of \kw{data imbalance} by training on synthetic data generated from a   Zipf-\ac{pcfg} with controlled imbalance levels. 
We find that Muon's \ac{nds} advantage over Adam becomes larger as the dataset becomes more imbalanced. 
Complementing the data perspective, we then turn to the model-structure perspective and decompose \ac{nds} into \kw{within-layer} and \kw{cross-layer} contributions. 
This decomposition shows that Muon's cross-layer contribution decreases rapidly during training, so in the middle and late stages of pretraining, its smaller \ac{nds} is mainly sustained by a smaller within-layer \ac{nds}.

To theoretically understand these observations, we study Muon's behavior on stylized quadratic optimization problems designed to mirror the landscape characteristics of \ac{llm} training. 
On these problems, Adam and GD empirically exhibit comparable \ac{nds} and loss decreases, so we focus the analysis on GD and Muon for simplicity. 
Under heterogeneous curvatures and gradient alignment with high-curvature directions, we prove that Muon has a smaller average \ac{nds} than GD because its update \kw{balances high- and low-curvature directions} more evenly. 
Moreover, when curvature heterogeneity is sufficiently strong, Muon achieves a smaller loss than GD after the same number of steps. 
These results provide theoretical support for our main empirical findings.\looseness=-1

Summarizing the empirical and theoretical findings, we identify a concrete mechanism behind Muon's advantage over Adam.
\begin{highlightbox}
Muon achieves a larger one-step loss decrease than Adam because its spectrally normalized update direction has lower \ac{nds}, thereby incurring a smaller second-order curvature cost.
\end{highlightbox}

\section{Related Work}
\label{app:related-work}

\paragraph{Muon and structured matrix updates.}
Recent works study Muon from empirical, algorithmic, and geometric perspectives.
\citet{jordan6muon} first established Muon's gains on matrix-valued parameters, while \citet{shah2025practical,liu2025muon,sato2025analysis,li2025note,kim2026convergence,si2025adamuon} examine its practical behavior, scalability, convergence properties, and algorithmic variants.
One line of follow-up work asks what geometry Muon's orthogonalized update is implicitly optimizing, as reflected by \citet{chen2025muon,ma2026preconditioning,shulgin2025beyond,kim2026convergence}.
Another line studies how to make orthogonalized matrix updates cheaper, more scalable, or more robust in practical training systems.
This includes distributed, periodic, low-rank, adaptive, federated, quantized, and long-tailed variants such as~\citet{ahn2025dion,khaled2025muonbp,si2025adamuon,he2025low,li2025normuon,liu2025fedmuon,zhang2026teon,zhang2026muon+,liu2026muon,cheng2026trasmuon}. Because Muon's practical implementation depends on approximate polar or Newton--Schulz maps, work on faster orthogonalization, including \citet{grishina2025accelerating,hu2026unso,boissin2025turbo}, is also closely connected.
At a broader level, \citet{bernstein2024old,kovalev2025understanding,pethick2025training,lau2025polargrad,an2025asgo,wen2025fantastic} place Muon and related orthogonalized or structured updates inside a larger geometric family.
This broader structured-update viewpoint also connects Muon to~\citet{bernstein2024modular,pethick2025training,lau2025polargrad,an2025asgo}, and matrix- or block-aware preconditioning methods such as~\citet{vyas2024soap,shazeer2018adafactor,anil2020scalable,martens2015optimizing,gupta2018shampoo}. There is also a line of work understanding and explaining why Muon outperforms Adam as \citet{wang2025muon, li2026muon, kim2026sharp, vasudeva2025generalization, qi2026delving}. Our results are also complementary to the memory-centric explanation of \citet{wang2025muon}, which attributes Muon's advantage to more balanced tail-end associative memory learning under heavy-tailed distributions.
However, we identify a local curvature mechanism: at the same iterate, Muon can pay a smaller directional sharpness penalty than Adam along its matrix update.

\paragraph{Adam and adaptive optimization.}
Our comparison also connects to the large literature on Adam and adaptive gradient methods, beginning with \citep{kingma2014adam,duchi2011adaptive}.
This literature ranges from decoupled regularization and convergence theory, as in~\citet{loshchilov2017decoupled,reddi2019convergence,chen2018convergence,zhou2018convergence,defossez2020simple,guo2021novel,li2023convergence,zhang2022adam,zou2019sufficient, liu2019variance,zhuang2020adabelief,zhang2024adam}. Other works examine why adaptive methods behave differently from SGD or spectral-gradient methods in language modeling, imbalanced data, and Transformer training \citep{kunstner2024heavy,vasudeva2025generalization,pan2023toward,zhang2024transformers}. We do not analyze Adam's convergence in this work. We use it as the baseline and compare the curvature it encounters along its matrix update against the corresponding Muon update.\looseness=-1

\paragraph{Sharpness, curvature, and optimization geometry.}
A large body of literature studies the role of sharpness and curvature in deep learning optimization and generalization.
Early work connected large-batch training, minima geometry, and generalization, while also pointing out that naive sharpness notions can be sensitive to parameterization.
This line includes~\citet{keskar2016large,dinh2017sharp,li2018visualizing,izmailov2018averaging}. Another line studies sharpness as a training-dynamics or algorithm-design quantity,
including sharpness-aware minimization and its efficient variants~\citep{cohen2021gradient,foret2020sharpness,kwon2021asam,andriushchenko2023modern,wen2023sharpness,du2021efficient,du2022sharpness}.
Closer to our setting, recent Transformer and LLM studies connect curvature heterogeneity to optimization behavior, including~\citet{zhang2024transformers,wang2025sharpness,pan2023toward,kalra2026scalable}.
These works primarily analyze global sharpness, flatness, or curvature-aware training rules.
We adopt a local, optimizer-dependent viewpoint: because Muon and Adam induce different update matrices at the same parameter point, we compare the Hessian curvature along each update direction, which explains why the optimizer gap appears in the second-order term even when the first-order gains are similar.

\section{Preliminaries}

In this section, we introduce the details of Adam and Muon optimizers and establish the notation used throughout the paper.
 
\textbf{Adam} has been the default optimizer for training \acp{llm} over the past decade~\citep{kingma2014adam}. 
It normalizes parameter updates coordinate-wise using exponential moving averages of the first and second moments of the stochastic gradient. 
For a matrix parameter $W_t^{\adam} \in \mathbb{R}^{m \times n}$, let 
$G_t^{\adam}=\nabla_{W}\mathcal L_{D_t}(W_t^{\adam})$ denote the gradient on the mini-batch $D_t\subseteq\calD$, where $\mathcal L_{D_t}$ is the empirical training loss on $D_t$. 
Adam maintains the following states:
\begin{align*}
    M_t=\beta_1 M_{t-1}+(1-\beta_1)G_t^{\adam},\quad V_t=\beta_2 V_{t-1}+(1-\beta_2)(G_t^{\adam}\odot G_t^{\adam}),
\end{align*}
which estimate the first moment and the element-wise second moment of the stochastic gradient, respectively. 
Here, $\beta_1,\beta_2\in[0,1)$ are hyperparameters, and $\odot$ denotes the Hadamard product. 
After bias correction, $M'_t=M_t/(1-\beta_1^t)$ and $V'_t=V_t/(1-\beta_2^t)$, Adam uses the update direction
$Z_t^{\adam}=\eta_t M'_t/(\sqrt{V'_t}+\epsilon)$
and updates the parameter by
$W_{t+1}^{\adam}=W_t^{\adam}-Z_t^{\adam}$,
where the square root and division are applied element-wise.

\textbf{Muon} is a matrix-parameter optimizer that explicitly exploits the spectral structure of matrix gradients~\citep{jordan6muon}. 
This matrix-structure-aware design has been shown to outperform Adam in large-scale \ac{llm} pretraining~\citep{liu2025muon}. 
For a matrix parameter $W_t^{\muon}\in\bbR^{m\times n}$, let $G_t^{\muon}=\nabla_{W}\mathcal L_{D_t}(W_t^{\muon})$ denote the gradient on the mini-batch $D_t\subseteq\calD$. Muon maintains a momentum accumulator
$B_t=\mu B_{t-1}+G_t^{\muon}$, with $B_0=0$ and $\mu\in[0,1)$. 
Given the singular value decomposition $B_t=U_tS_tV_t^\top$, Muon spectrally normalizes the momentum matrix by setting $O_t=U_tV_t^\top$ and updates the parameter by
$$W_{t+1}^{\muon}=W_t^{\muon}-Z_t^{\muon}, \qquad 
\text{where} ~~Z_t^{\muon}=\eta_t O_t.$$
In practice, $O_t$ can be efficiently approximated by a small number of Newton--Schulz iterations, rather than computed by an exact singular value decomposition. 
The resulting update is scale-invariant: multiplying $B_t$ by any positive scalar does not change the update direction $Z_t^{\muon}$.

\paragraph{Notation.}
For a positive integer $N$, let $[N]=\{1,\ldots,N\}$.
For matrices $A, B \in \mathbb{R}^{m \times n}$, define the Frobenius inner product and the corresponding norm as $\langle A, B \rangle= \Tr(A^\top B), \|A\|_F= \sqrt{\langle A, A \rangle}.$
Throughout, $\mathcal{L}_D$ denotes the empirical training loss on $D$, which is twice continuously differentiable in the region of interest. $W$ denotes the matrix parameter. $G= \nabla_W \mathcal{L}_D(W)$ denotes the gradient, and $\mathcal{H}=\nabla_W^2 \mathcal{L}_D(W):\mathbb{R}^{m\times n}\rightarrow \mathbb{R}^{m\times n}$ denotes the Hessian operator acting on matrix perturbations. For any matrix perturbation $Z\in\mathbb{R}^{m\times n}$, we define $\mathcal H[Z]=\frac{d}{d\epsilon}
\nabla_W \mathcal L_D(W+\epsilon Z)|_{\epsilon=0}$.
We write $\mat(\mathcal{H})\in\mathbb{R}^{mn\times mn}$ for the matrix representation of $\mathcal{H}$ under vectorization, i.e., $\mathrm{vec}(\mathcal{H}[Z])=\mat(\mathcal{H})\,\mathrm{vec}(Z)$.
For a vector $x=(x_1,\ldots,x_d)$, $\Diag(x_1,\ldots,x_d)$ denotes the $d\times d$ diagonal matrix with entries $x_1,\ldots,x_d$.

\section{Main Results}
\label{sec:main-results}

In this section, we present the main empirical findings that characterize Muon's advantage over Adam from a curvature perspective. In Section~\ref{sec:one-step-decrease}, we show that Muon achieves a larger one-step loss decrease than Adam, and that this advantage is driven by a smaller second-order curvature penalty rather than a larger first-order gain. In Section~\ref{sec:sharpness-comparison}, we decompose the curvature penalty and identify lower \ac{nds} as the source of Muon's smaller curvature cost. In Section~\ref{sec:imbalance-results}, we examine how dataset imbalance amplifies the \ac{nds} gap between the two optimizers. Finally, in Section~\ref{sec:inner-cross-results}, we decompose \ac{nds} into within-layer and cross-layer contributions to understand how different layers contribute to Muon's curvature advantage.

\subsection{Muon Incurs a Smaller Second-Order Curvature Penalty Than Adam}
\label{sec:one-step-decrease}

To study the superiority of Muon over Adam from the curvature perspective, we begin with the local decomposition of one-step optimization progress. 
Specifically, for the parameter matrix $W$ and an update $Z$, the empirical loss decrease $\Delta_{D}(W,Z)=\mathcal{L}_{D}(W)-\mathcal{L}_{D}(W - Z)$ on the mini-batch $D$ can be approximated via 
\begin{equation}
\Delta_{D}(W,Z)\approx
\langle G, Z \rangle
- \frac{1}{2} \langle Z, \mathcal H[Z] \rangle =I_D^{(1)}(W,Z)-I_D^{(2)}(W,Z).
\label{eq:taylor}
\end{equation}
This equation decomposes the loss decrease $\Delta_{D}(W,Z)$ into the first-order decrease $I_D^{(1)}(W,Z)=\langle G, Z \rangle$ and the curvature penalty $I_D^{(2)}(W,Z)=1/2\cdot\langle Z, \mathcal H[Z] \rangle$. 
The first-order term measures the loss reduction induced by moving along the update direction, whereas the curvature penalty captures the second-order increase in loss that offsets this reduction. We refer to the right-hand side of Eqn.~\eqref{eq:taylor} as the predicted loss decrease, and the left-hand side $\Delta_{D}(W,Z)$ as the realized loss decrease.

\begin{figure}[t]
\centering
\subfigure[Predicted vs.\ realized loss decreases.\label{fig:descent-comparison-a}]{
\includegraphics[width=0.31\textwidth]{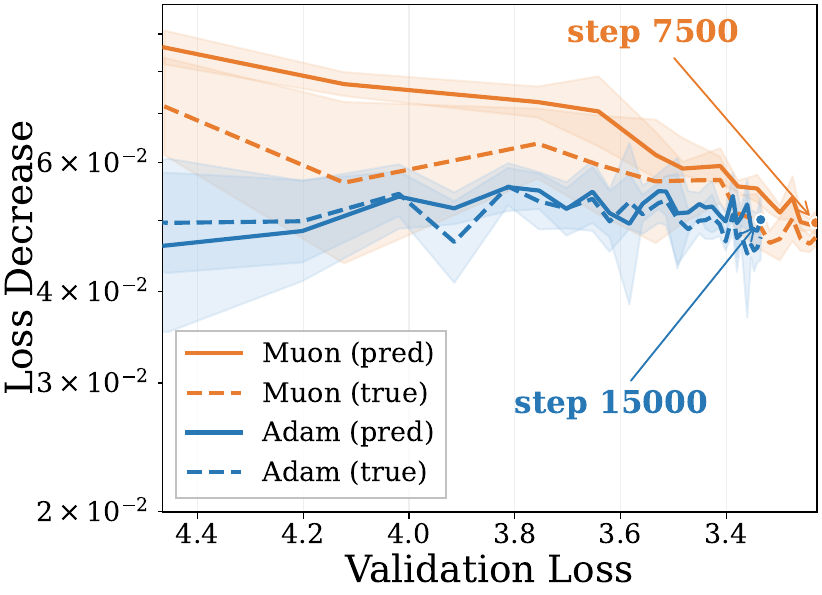}
}
\subfigure[First-order decrease.\label{fig:descent-comparison-b}]{
\includegraphics[width=0.31\textwidth]{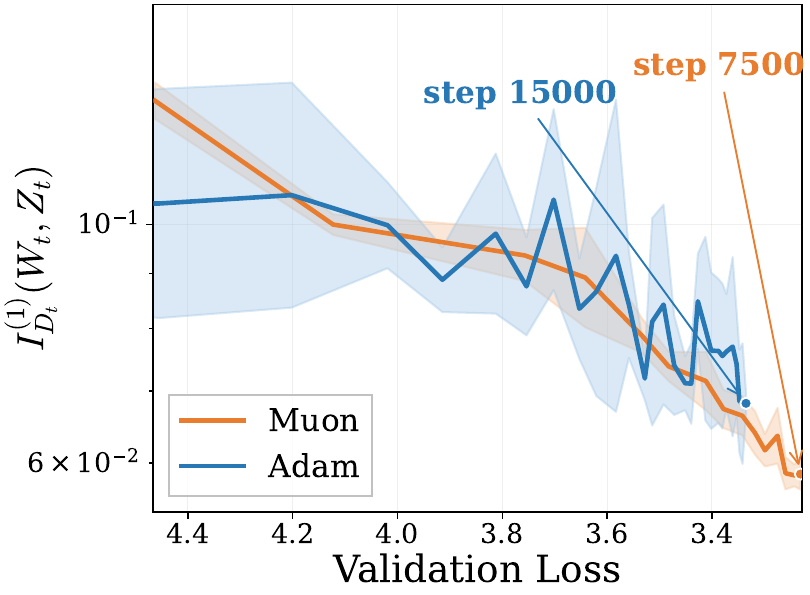}
}
\subfigure[Curvature penalty.\label{fig:descent-comparison-c}]{
\includegraphics[width=0.31\textwidth]{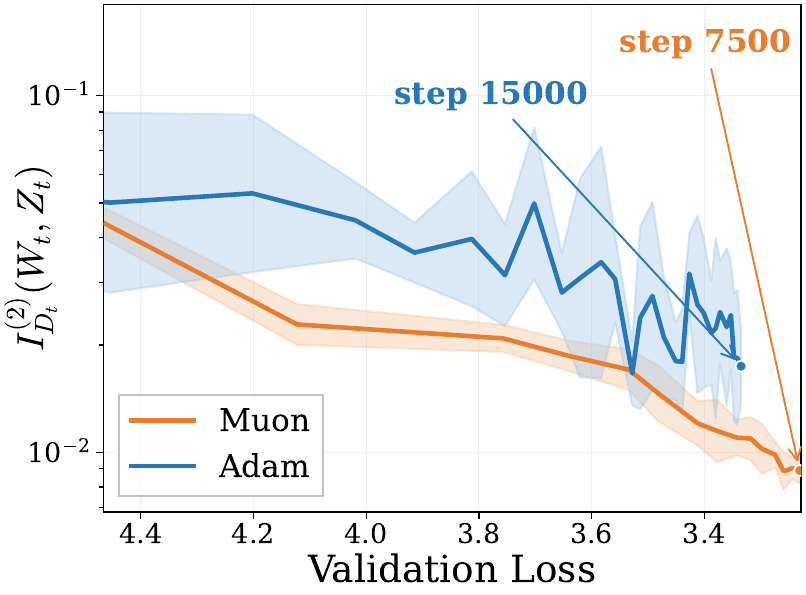}
}
\caption{Decomposition of one-step optimization progress along the update directions of Muon and Adam.
Panel (a) compares the predicted one-step loss decrease, $I_{D_t}^{(1)}(W_t,Z_t)-I_{D_t}^{(2)}(W_t,Z_t)$, with the realized one-step loss decrease, $\Delta_{D_t}(W_t,Z_t)$.
Panel (b) reports the first-order decrease $I_{D_t}^{(1)}(W_t,Z_t)$, and Panel (c) reports the curvature penalty $I_{D_t}^{(2)}(W_t,Z_t)$.
The results show that Muon and Adam achieve similar first-order decreases, while Muon incurs a smaller curvature penalty.}
\vspace{-1em}
\label{fig:descent-comparison}
\end{figure}

\vspace{5pt}
\noindent \textbf{Experiment setups.} 
To evaluate whether the approximation in Eqn.~\eqref{eq:taylor} explains the advantage of Muon over Adam, we compute $\Delta_{D}(W,Z)$, $I_D^{(1)}$, and $I_D^{(2)}$ under the following setting. We train a 124M-parameter NanoGPT model on the FineWeb dataset~\citep{penedo2024fineweb}. 
We emphasize this scale because 124M parameters is among the largest model sizes considered in existing studies of curvature in \ac{llm} pretraining~\citep{zhang2024transformers,dong2025towards}, as Hessian-based computation scales quadratically with the parameter dimension. 
For both Adam and Muon, we select the optimal learning rate by grid search.  Full experimental details are provided in Appendix~\ref{app:exp-details}. 
At step $t$ of Muon's trajectory, given the training batch $D_t$, the current parameter $W_t^{\muon}$, and the update direction $Z_t^{\muon}$, we compute $\Delta_{D_t}(W_t^{\muon},Z_t^{\muon})$, $I_{D_t}^{(1)}(W_t^{\muon},Z_t^{\muon})$, and $I_{D_t}^{(2)}(W_t^{\muon},Z_t^{\muon})$.
We compute the same quantities for Adam along its own optimization trajectory using its corresponding parameter updates. 
Since Muon reaches lower loss than Adam at the same training step, for a fair comparison, we compare the two optimizers at matched validation loss rather than at matched training step.

\vspace{5pt}
\noindent \textbf{Experiment findings.} 
In Figure~\ref{fig:descent-comparison-a}, we plot the predicted loss decrease 
$I_{D_t}^{(1)}(W_t^{\opt},Z_t^{\opt})-I_{D_t}^{(2)}(W_t^{\opt},Z_t^{\opt})$ 
and the realized loss decreases 
$\Delta_{D_t}(W_t^{\opt},Z_t^{\opt})$ for $\opt\in\{\mathsf{Muon},\mathsf{Adam}\}$ at matched validation-loss levels. 
Across validation-loss levels, Muon achieves a larger realized loss decrease than Adam, consistent with its superior training efficiency. 
For Adam, the predicted loss decrease closely matches the realized loss decrease. 
For Muon, the predicted loss decrease is slightly smaller than the realized loss decrease. 
Thus, the following analysis explains Muon's advantage over Adam within the second-order approximation in Eqn.~\eqref{eq:taylor}, while leaving open the possibility that higher-order effects further contribute to  Muon's advantage.

In Figures~\ref{fig:descent-comparison-b} and \ref{fig:descent-comparison-c}, we report the first-order decrease and curvature penalty of Adam and Muon at matched validation-loss levels. 
As shown in Figure~\ref{fig:descent-comparison-b}, Adam and Muon have \kw{comparable first-order decreases} throughout optimization, as the two curves remain at similar levels across all validation-loss values, although Adam exhibits higher variability across seeds. 
In contrast, Figure~\ref{fig:descent-comparison-c} shows a clear gap in the curvature term: {\color{figureblue}Adam's curve (blue)} lies consistently above {\color{figureorange}Muon's curve (orange)}, indicating that Muon incurs a substantially \kw{smaller Hessian quadratic-form penalty} along its update directions. 
Together with Figure~\ref{fig:descent-comparison-a}, these results show that the larger one-step loss decrease of Muon is driven primarily by a smaller second-order curvature penalty, rather than by a larger first-order gain. This leads to our first observation.

\begin{highlightbox}
    {\bf Observation 1:} Muon achieves a larger one-step loss decrease than Adam under validation-loss alignment. The gap is primarily due to a smaller second-order curvature cost.
\end{highlightbox}

\subsection{Muon's Smaller Curvature Penalty Comes from Its Update Direction}\label{sec:sharpness-comparison}

Observation~1 shows that Muon's advantage in one-step loss decrease is primarily explained by its smaller second-order curvature cost.
Note that the curvature penalty $I^{(2)}_D(W,Z) = 1/2\cdot\langle Z, \mathcal{H}[Z]\rangle$ depends on both the scale of the update and the curvature along its direction. We now ask: does Muon's smaller curvature cost arise because it takes smaller steps, or because its update direction encounters less curvature? 
To disentangle the effects of update direction and update scale, we measure the curvature of the local quadratic form by normalizing out the scale. 
Specifically, following \citet{pan2023toward}, we define the \acf{nds} along a nonzero update $Z$ as
\begin{equation}
    \Sc_F(W;Z)=\langle Z, \mathcal{H} [Z] \rangle/\|Z\|_F^2. 
    \label{eq:taylor_sharpness}
\end{equation}
With this definition, the curvature penalty factorizes as $I^{(2)}_D(W,Z) =1/2 \cdot \|Z\|_F^2 \cdot \Sc_F(W;Z)$, so a smaller curvature cost can \emph{only} come from a smaller update norm or a lower \ac{nds}.
We therefore examine whether Muon's smaller curvature penalty is driven by a smaller update norm or by a smaller \ac{nds}. 

Similar to Section~\ref{sec:one-step-decrease}, we plot the update norm and \ac{nds} at matched validation-loss levels. To compute the ratios between Adam and Muon, we linearly interpolate the values between adjacent checkpoints, since the two optimizers may not reach exactly the same validation loss at a recorded step. We also report the corresponding ratios at matched training steps in Appendix~\ref{app:sharpness}
, which leads to the same conclusion.


\vspace{5pt}
\noindent \textbf{Experiment findings.} In Figures~\ref{fig:sharpness-valloss} and \ref{fig:sharpness-frob-a}, we plot the values of $\|Z_t\|_{F}$ and $\calS_F$ aligned according to the validation loss. 
As shown in Figure~\ref{fig:sharpness-valloss}, {\color{figureblue}Adam's curve (blue)} lies consistently above {\color{figureorange}Muon's curve (orange)}, indicating that Muon has \kw{lower \ac{nds}} than Adam throughout training. In contrast, Figure~\ref{fig:sharpness-frob-a} shows that the two  optimizers have \kw{comparable update norms}: both curves remain nearly flat and close together across all validation-loss levels. Figure~\ref{fig:sharpness-frob-b} confirms this decomposition quantitatively by plotting the Adam-to-Muon ratios: {\color{figuregreen}the update-scale $\|Z_t\|_F^2$ ratio (green dotted)} stays close to $1$, while {\color{figurered}the curvature-penalty $I^{(2)}_{D_t}$ ratio (dark red solid)} and {\color{figurelightblue}the \ac{nds} $\Sc_F$ ratio (cyan dashed)} track each other closely, with an average \ac{nds} ratio of $1.76$. This shows that the curvature-penalty gap is almost entirely accounted for by the \ac{nds} gap, not by differences in update scale.
Thus, we reach the following conclusion.

\begin{figure}[t]
\centering
\subfigure[Sharpness comparison.\label{fig:sharpness-valloss}]{
\includegraphics[width=0.31\textwidth]{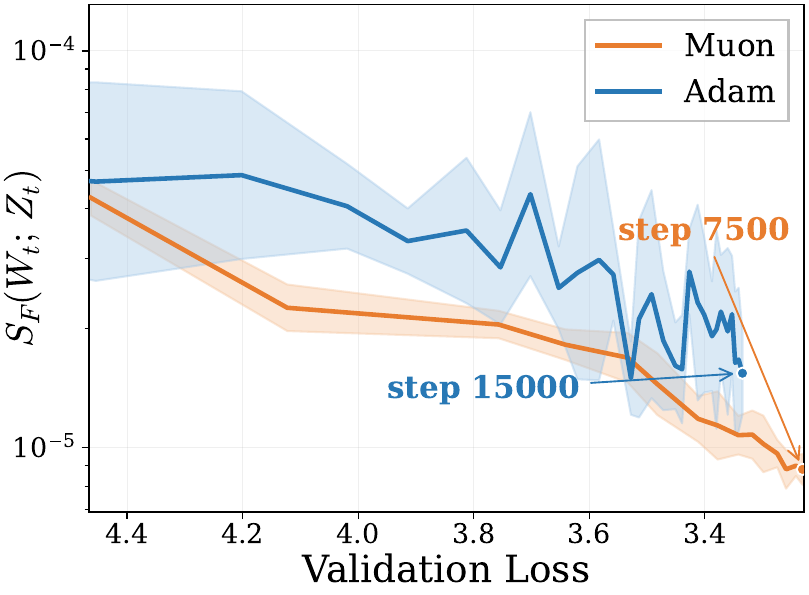}
}
\subfigure[Update norm comparison. \label{fig:sharpness-frob-a}]{
\includegraphics[width=0.31\textwidth]{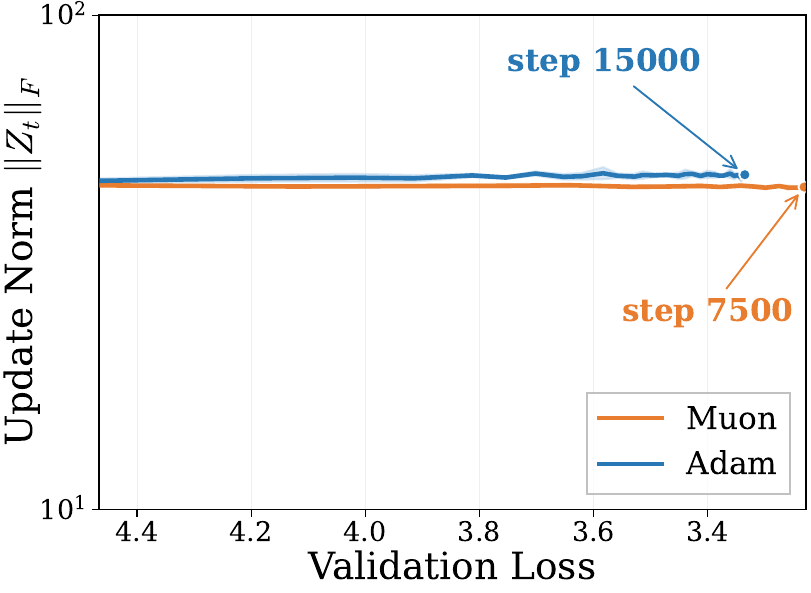}
}
\subfigure[Decomposition of the Adam-to-Muon sharpness gap.\label{fig:sharpness-frob-b}]{
\includegraphics[width=0.31\textwidth]{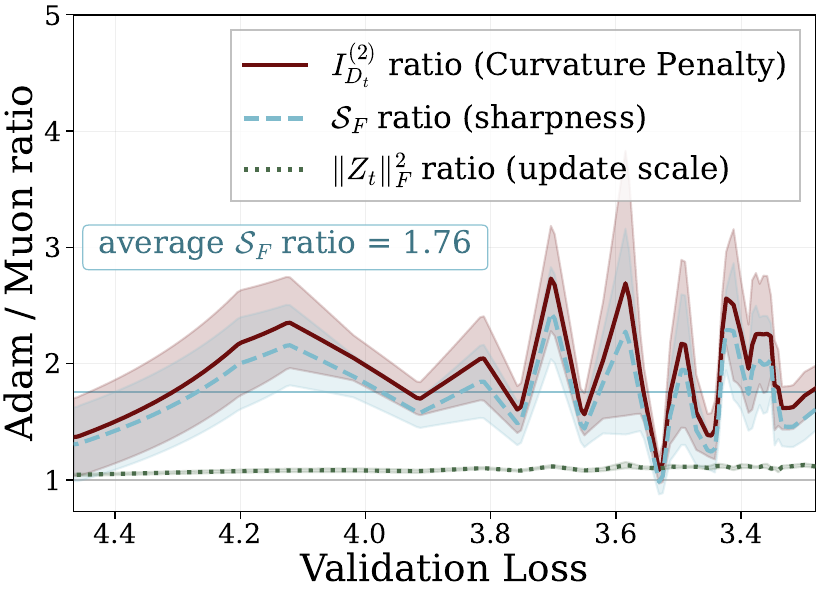}
}
\caption{\ac{nds} and update-norm comparisons between Muon and Adam.
Panel (a) plots \ac{nds} for Muon and Adam.
Panel (b) plots the update norms of Muon and Adam.
Panel (c) reports the Adam-to-Muon ratios of the curvature penalty, \ac{nds}, and the squared Frobenius norm of the update.
Muon and Adam have similar update norms, whereas Muon has smaller \ac{nds} than Adam. Moreover, the Adam-to-Muon ratio of \ac{nds} closely tracks that of the curvature penalty.}
\label{fig:sharpness-frob}
\end{figure}

\begin{highlightbox}
    {\bf Observation 2:} Muon and Adam have comparable update norms, so Muon's smaller curvature penalty is driven by the notably smaller \ac{nds}.
\end{highlightbox}
Observations~1 and~2 show that Muon's larger one-step loss decrease is driven by a smaller second-order curvature penalty, which in turn arises from lower \ac{nds} rather than smaller update norms. 
We next investigate how this curvature advantage depends on two key factors: the training data and the model architecture. 
On the data side, we identify which data properties affect the magnitude of Muon's curvature advantage. 
On the model side, we examine how the interaction between different layers contributes to this advantage.

\subsection{Dataset Imbalance Widens the \ac{nds} Gap between Adam and Muon}
\label{sec:imbalance-results}

Prior work suggests that the tail structure and imbalance of training data can interact strongly with optimizer behavior in two ways. 
First, Hessian analyses of neural networks show that the Hessian spectrum depends sensitively on the data mixture~\citep{sagun2017empirical,papyan2018full}. 
Second, recent studies find that Muon can outperform Adam especially on heavy-tailed data~\citep{wang2025muon,vasudeva2025muon}. 
Motivated by these findings, we study whether dataset imbalance amplifies the normalized directional-sharpness gap identified above.

\begin{figure}[t]
\centering
\subfigure[$\widetilde{S}_{\opt}(s)$ at $s=0, 0.5,  1$.\label{fig:imbalance-sharpness-a}]{
\begin{minipage}[b]{0.47\textwidth}
\centering
\includegraphics[width=\linewidth]{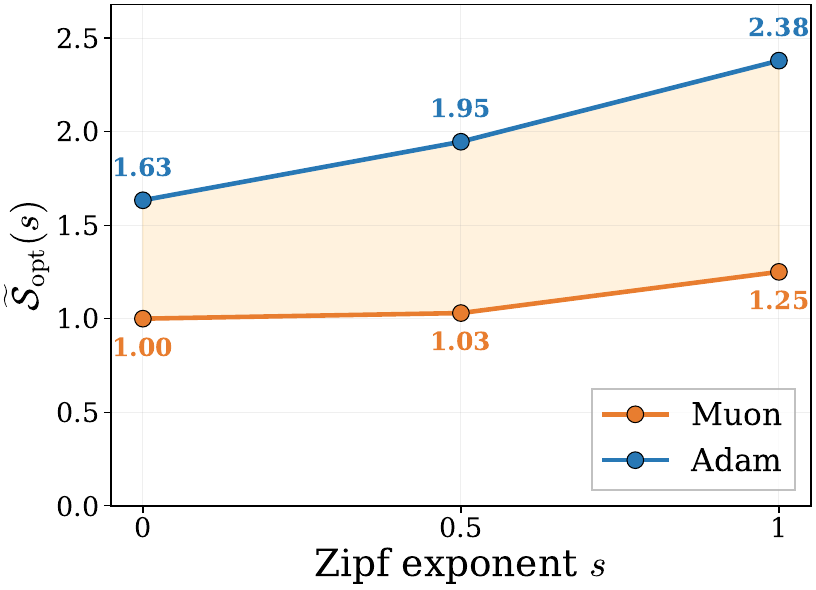}
\end{minipage}
}
\subfigure[$\Delta(s)$ at $s=0, 0.5,  1$\label{fig:imbalance-sharpness-b}]{
\begin{minipage}[b]{0.47\textwidth}
\centering
\includegraphics[width=\linewidth]{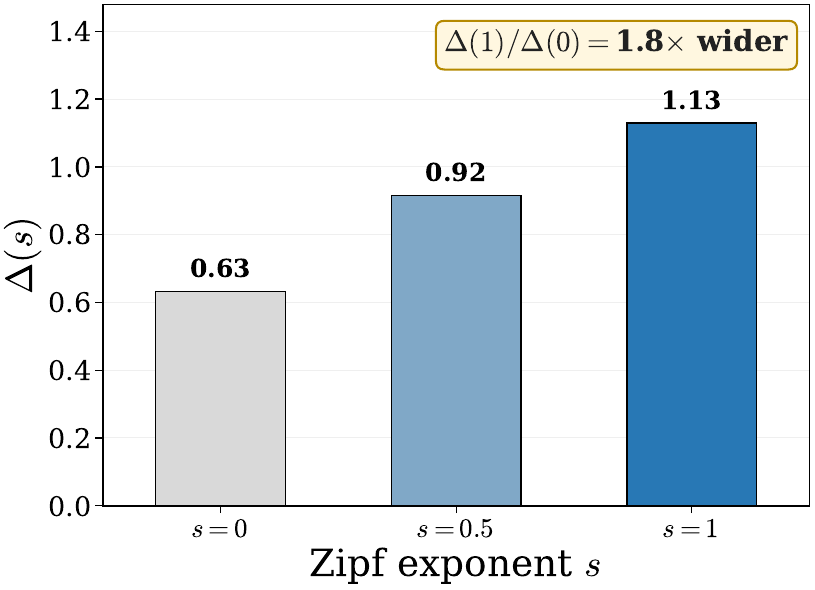}
\end{minipage}
}
\caption{Effect of different levels of imbalance ($s=0, 0.5, 1$) on \ac{nds}. Panel~(a) reports the trajectory-averaged \ac{nds}, normalized by Muon's value at $s=0$, under three imbalance levels ($s=0, 0.5, 1$).
Panel~(b) reports the Adam--Muon gap in \ac{nds} under the same settings.
The results show that Muon's advantage over Adam in \ac{nds} becomes larger as the data become more imbalanced.}
\label{fig:imbalance-sharpness}
\end{figure}

\vspace{5pt}
\noindent \textbf{Experiment setup: data generation.}
To study how data imbalance affects \ac{nds}, we construct synthetic training data using a Zipf-\ac{pcfg}, which allows us to explicitly control the degree of imbalance.
Specifically, we instantiate a Zipf-\ac{pcfg} with topics $k\in[K]$ and partition the vocabulary $\calV$ into $C$ token classes $\calV_1,\dots,\calV_C$, so that $\calV=\cup_{c=1}^C\calV_c$.
Different token classes correspond to different sentence components, such as nouns, adjectives, and verbs.
Each topic $k$ has its own preference distribution within each token class $c$.
For example, the topic $k=\text{``food''}$ assigns higher probabilities to eating-related verbs.
Given topic $k$ and token class $c$, let $\phi_{k,c}(j)$ denote the base probability of token $j\in\calV_c$.
To introduce controllable imbalance, we rank tokens in $\calV_c$ in decreasing order of $\phi_{k,c}(j)$ and denote the rank of token $j$ by $r(j,c,k)$, where the most probable token has rank $1$ and $r(j,c,k)\in\{1,\dots,|\calV_c|\}$.
At imbalance level $s$, we sample tokens from the reweighted distribution proportional to $r(j,c,k)^{-s}\phi_{k,c}(j)$. We then train a 9M-parameter NanoGPT model (4 layers, 4 attention heads, model dimension 256) on datasets generated under imbalance levels $s\in\{0, 0.5, 1\}$ for $10{,}000$ steps, using both Adam and Muon with learning rates selected by grid search. Additional experimental details are provided in Appendix~\ref{app:exp-details}.

\vspace{5pt}
\noindent \textbf{Metric: trajectory-averaged NDS.}
To evaluate \ac{nds} along the full training trajectory, for each optimizer $\opt\in\{\mathsf{Muon},\mathsf{Adam}\}$, we define the trajectory-averaged \ac{nds} at imbalance level $s$ as $$\bar S_{\opt}(s)=\sum_{t\in\calT} S_F(W_t^{\opt,s}; Z_t^{\opt,s})/|\calT|, $$ where $\calT$ denotes the set of training steps, $s\in\{0,0.5,1\}$ is the imbalance level, and $W_t^{\opt,s}$ and $Z_t^{\opt,s}$ denote the parameter and update induced by optimizer $\opt$ under imbalance level $s$ at step $t$.
To highlight the difference between Adam and Muon, we normalize $\bar S_{\opt}(s)$ by Muon's value at $s=0$: $\tilde S_{\opt}(s)=\bar S_{\opt}(s)/\bar S_{\mathsf{Muon}}(0)$.
We further define the normalized sharpness gap at imbalance level $s$ as $\Delta(s)=\tilde S_{\mathsf{Adam}}(s)-\tilde S_{\muon}(s)$.

\vspace{5pt}
\noindent \textbf{Experiment findings.}
In Figure~\ref{fig:imbalance-sharpness-a} we plot  the normalized trajectory-averaged \ac{nds} $\tilde S_{\opt}(s)$ against the  Zipf exponent   $s \in \{0,\,0.5,\,1\}$.
As shown in this figure, both optimizers' \ac{nds} increases monotonically with imbalance, but the \kw{effect is far stronger for Adam}:
Adam's normalized  \ac{nds} rises from $1.63$ to $2.38$ as $s$ increases from $0$ to $1$, whereas Muon's grows only from $1.00$ to $1.25$. The widening shaded region between the two curves reflects the growing gap. In addition, Figure~\ref{fig:imbalance-sharpness-b} quantifies this gap  directly: $\Delta(s)$ \kw{widens monotonically} from $0.63$ at $s=0$ to $1.13$ at $s=1$, a $1.8\times$ increase as the data becomes more imbalanced.
We summarize this observation as follows.

\begin{highlightbox}
    {\bf Observation 3:} Increasing imbalance level of the dataset not only amplifies the \ac{nds} for both Muon and Adam, but  also widens the \ac{nds} gap between them.
\end{highlightbox}

\subsection{Muon's \ac{nds} increasingly shifts toward within-layer Hessian blocks}
\label{sec:inner-cross-results}

Observation~2 establishes a normalized directional-sharpness gap between Muon and Adam over all model parameters.
In this section, we study how different layers contribute to this gap.

\vspace{5pt}
\noindent \textbf{Experiment setup: within-/cross-layer decomposition.}
Consider a model with $L$ layers, where the full parameter at step $t$ is $W_t=(W_{t,1},\ldots,W_{t,L})$ with $W_{t,\ell}\in\mathbb{R}^{m_\ell\times n_\ell}$ denoting the weight matrix of layer $\ell$. The corresponding update decomposes as $Z_t=(Z_{t,1},\ldots,Z_{t,L})$, where $Z_{t,\ell}\in\mathbb{R}^{m_\ell\times n_\ell}$ is the update to layer $\ell$ produced by the optimizer.
The Hessian operator $\mathcal{H}$ can likewise be decomposed into layer-wise blocks: for layers $\ell,\ell'\in[L]$, let $\mathcal{H}_{\ell\ell'}$ denote the block that maps a perturbation at layer $\ell'$ to the resulting second-order effect at layer $\ell$. By definition, the diagonal blocks $\mathcal{H}_{\ell\ell}$ capture within-layer curvature, while the off-diagonal blocks $\mathcal{H}_{\ell\ell'}$ ($\ell\neq\ell'$) capture cross-layer interactions.

This block structure allows us to decompose the \ac{nds} into within-layer and cross-layer contributions:
$\Sc_F(W_t;Z_t)=\Sc_F^{\mathrm{within}}(W_t;Z_t)+\Sc_F^{\mathrm{cross}}(W_t;Z_t)$,
where
\begin{align*}
    \Sc_F^{\mathrm{within}}(W_t;Z_t)=\sum_{\ell=1}^{L}\langle Z_{t,\ell},\mathcal H_{\ell\ell}[Z_{t,\ell}]\rangle/\|Z_t\|_F^2,\quad \Sc_F^{\mathrm{cross}}(W_t;Z_t)=\sum_{\ell\neq \ell^{\prime}}\langle Z_{t,\ell},\mathcal H_{\ell\ell^{\prime}}[Z_{t,\ell^{\prime}}]\rangle/\|Z_t\|_F^2.
\end{align*}
Intuitively, $\Sc_F^{\mathrm{within}}$ measures the curvature encountered when each layer's update interacts only with its own Hessian block, while $\Sc_F^{\mathrm{cross}}$ captures the additional curvature arising from interactions between updates at different layers.
We further define the relative within-layer contribution as $\rho_t^{\mathrm{within}}=\Sc_F^{\mathrm{within}}(W_t;Z_t)/\Sc_F(W_t;Z_t)$.

\begin{figure}[t]
\centering
\subfigure[\ac{nds} of within-layer and cross-layer components\label{fig:inner-a}]{
\begin{minipage}[b]{0.42\textwidth}
\centering
\includegraphics[width=\linewidth]{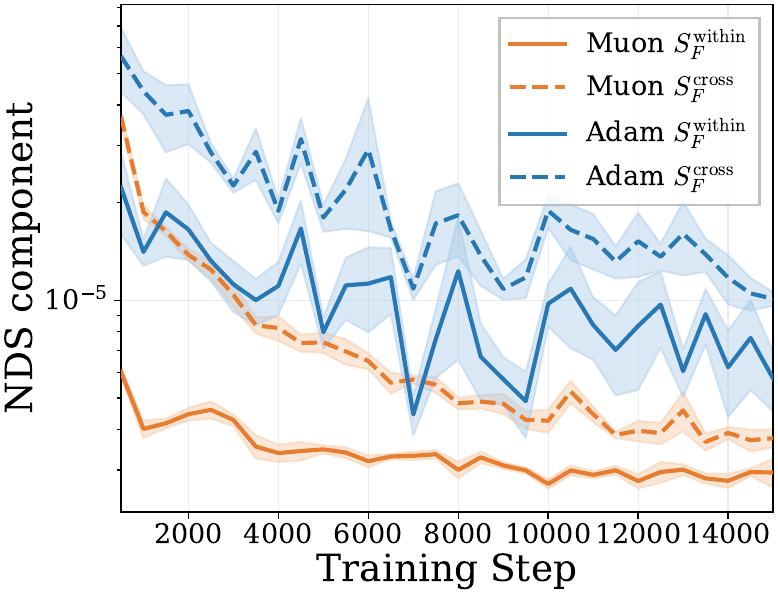}
\end{minipage}
}
\subfigure[Within-layer fraction\label{fig:inner-b}]{
\begin{minipage}[b]{0.42\textwidth}
\centering
\includegraphics[width=\linewidth]{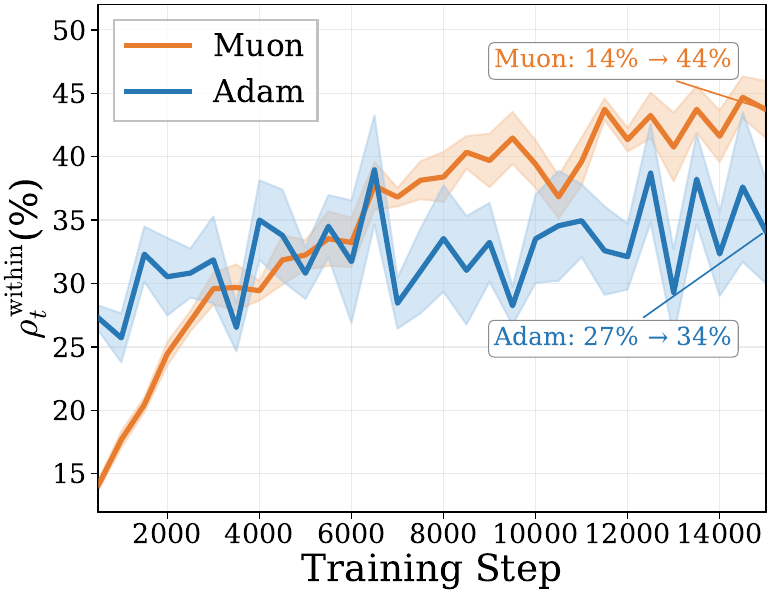}
\end{minipage}
}
\caption{Within-layer and cross-layer decomposition of directional sharpness over training.
Panel~(a) reports the within-layer and cross-layer components of $\Sc_F(W_t;Z_t)$ for Muon and Adam.
Panel~(b) reports the within-layer fraction.
The results show that, for Muon, the within-layer component accounts for an increasingly large share of directional sharpness over the course of training.
}
\label{fig:inner-cross-time}
\end{figure}

\vspace{5pt}
\noindent \textbf{Experiment findings.}
Figure~\ref{fig:inner-a} plots $\Sc_F^{\mathrm{within}}$ (solid lines) and $\Sc_F^{\mathrm{cross}}$ (dashed lines) for Adam (blue) and Muon (orange) under the same experimental setting as Section~\ref{sec:one-step-decrease}.
All four curves decrease over training, and both components are consistently smaller for Muon than for Adam: Muon's curves (orange) lie below Adam's curves (blue) throughout.
For Adam, the solid and dashed blue curves decrease at comparable rates, maintaining a roughly stable ratio between the two components.
For Muon, however, the dashed orange curve ($\Sc_F^{\mathrm{cross}}$) drops much faster than the solid orange curve ($\Sc_F^{\mathrm{within}}$), so the two Muon curves converge over training. This indicates that the \kw{within-layer component becomes increasingly dominant} in Muon's \ac{nds} as training proceeds.

Figure~\ref{fig:inner-b} quantifies this trend by plotting the within-layer fraction $\rho_t^{\mathrm{within}}$. {\color{figureorange}Muon's curve (orange)} rises steeply from about $14\%$ early in training to about $44\%$ later in training, nearly tripling its within-layer share. By contrast, {\color{figureblue}Adam's curve (blue)} fluctuates around $30\%$ with only a modest increase from about $27\%$ to about $34\%$.
This highlights the importance of Muon's small $\Sc_F^{\mathrm{within}}$ in keeping the full-model \ac{nds} low in the middle and late training stages, while Adam maintains a comparatively stable balance between within-layer and cross-layer contributions.
Moreover, Appendix~\ref{app:layer-localization} shows that the within-layer Adam--Muon sharpness gap is not uniformly distributed across layers, with nearly all of the gap concentrated in the first and deepest layers.\looseness=-1
\begin{highlightbox}
    {\bf Observation 4:} Over training, Muon's directional sharpness shifts increasingly toward within-layer Hessian blocks, while Adam's sharpness composition remains comparatively stable. Both within-layer and cross-layer components remain smaller for Muon than for Adam.
\end{highlightbox}

\section{A Case Study of  Structured Matrix-Block Quadratic Models}
\label{sec:case-study-kronecker}

Section~\ref{sec:main-results} establishes that Muon achieves a larger one-step loss decrease than Adam due to lower \ac{nds} (Observations~1--2), that data imbalance amplifies this gap (Observation~3), and that Muon's \ac{nds} advantage increasingly concentrates in within-layer Hessian blocks as training proceeds (Observation~4).
In this section, we provide theoretical justification for these empirical findings.
Since Observation~4 shows that within-layer curvature becomes the dominant component of Muon's \ac{nds} advantage, we isolate a single weight matrix and study the local curvature within this block on a quadratic model, comparing the curvature encountered by update directions induced by different optimizers.

\subsection{Structured Quadratic Model}\label{sec:model}

Motivated by the second-order Taylor approximation in Eqn.~\eqref{eq:taylor}, we focus on the local quadratic landscape around one optimization step in \ac{llm} pretraining. 
Given a fixed parameter $W_0\in\bbR^{d_1\times d_2}$, its gradient $G=\nabla\calL(W_0)$, and Hessian operator $\mathcal H=\nabla^2\calL(W_0)$, we consider the following quadratic model for an update $Y\in\bbR^{d_1\times d_2}$:
\begin{align}
\mathcal Q(Y)
    =\calL(W_0)-\langle G,Y\rangle+1/2\cdot\langle Y,\mathcal H[Y]\rangle .
\label{eq:quadratic-model}
\end{align}
To make this model representative of \ac{llm} pretraining, we impose four assumptions on the gradient and Hessian structure. Each assumption is empirically verified in real pretraining dynamics: we provide in-text verification figures below and additional details in Appendix~\ref{app:verify-relaxed-alignment}. We begin with a decomposition of the Hessian.

\begin{assumption}[Hessian Low Kronecker-Rankness]
\label{ass:block-curvature}
The local Hessian operator $\calH$ on the matrix block has small Kronecker rank. 
Specifically, let $\mat(\calH)\in\bbR^{d_1d_2\times d_1d_2}$ denote the matrix representation of $\calH$ under vectorization. 
Then there exist an integer $r \ll \min\{d_1^2,d_2^2\}$, symmetric matrices $A_k\in\bbR^{d_1\times d_1}$ and $B_k\in\bbR^{d_2\times d_2}$ such that
$\mat(\calH)=\sum_{k=1}^r B_k^\top\otimes A_k$,
where $\otimes$ denotes the Kronecker product.
\end{assumption}
This assumption states that the Hessian matrix admits a low Kronecker-rank approximation with rank $r$. 
It is motivated by prior work on K-FAC~\citep{martens2015optimizing,george2018fast}, which shows that the Fisher matrix can be effectively approximated by Kronecker-structured curvature factors. 
Since the Fisher matrix is closely related to the Hessian under standard conditions~\citep{wang2023theoretical}, these results suggest that Kronecker structure can also provide a useful approximation to Hessian curvature. 
We further verify this assumption in \ac{llm} pretraining with Muon.
Let $H_r\in\bbR^{d_1d_2\times d_1d_2}$ denote the best rank-$r$ Kronecker approximation of $\mat(\calH)$ in Frobenius norm. Figure~\ref{fig:kronecker-assumption-a} plots the explained Frobenius energy $\xi(r)=\|H_r\|_F^2/\|\mat(\calH)\|_F^2$ as a function of Kronecker rank $r$ for the four attention matrices ($W_Q$, $W_K$, $W_V$, $W_O$). All four curves rise steeply at small $r$ and saturate quickly: $W_K$ and $W_V$ cross the $\xi (r)=0.8$ threshold by $r\approx 3$--$5$, while $W_Q$ and $W_O$ reach it by $r\approx 7$--$9$.
Figure~\ref{fig:kronecker-assumption-b} visualizes the full Hessian $\mat(\calH)$, its rank-$4$ Kronecker approximation $H_4$, and the residual $\mat(\calH)-H_4$ for $W_V$ as heatmaps. The approximation captures the dominant block structure visible in the original Hessian, while the residual is uniformly faint, confirming that most energy is captured at low rank.
These results support Assumption~\ref{ass:block-curvature}.

\begin{figure}[t]
\centering
\subfigure[Explained Frobenius energy ratio $\xi(r)$ for the attention matrices.\label{fig:kronecker-assumption-a}]{
\includegraphics[width=0.27\textwidth,trim=6 4 4 1,clip]{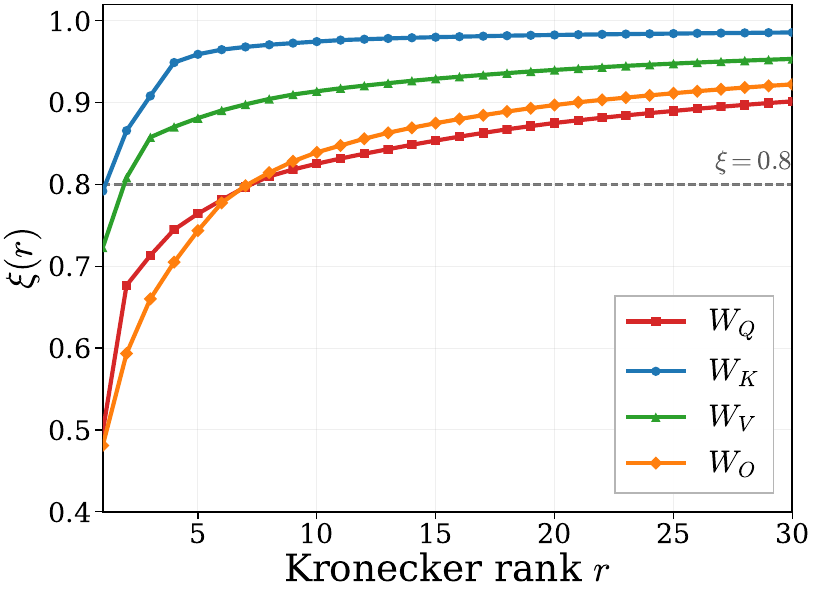}
}
\subfigure[$W_V$ Hessian, rank-$4$ Kronecker approximation, and residual.\label{fig:kronecker-assumption-b}]{
\includegraphics[width=0.68\textwidth]{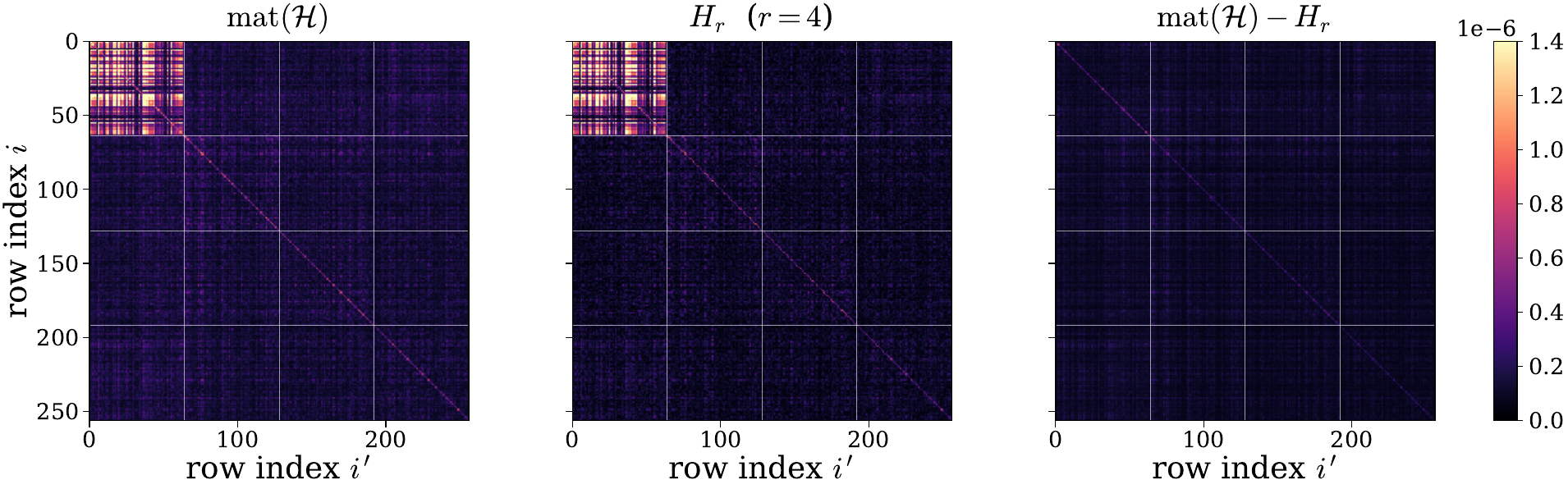}
}
\caption{Empirical support for Assumption~\ref{ass:block-curvature}.
Panel~(a) reports the fraction of Frobenius energy explained by low-rank Kronecker approximations to the Hessians of the four attention matrices.
Panel~(b) visualizes the $W_V$ Hessian, its rank-$4$ Kronecker approximation, and the residual error.
The results show that the Hessians of attention matrices can be well approximated by low-rank Kronecker products.
}
\label{fig:kronecker-assumption}

\end{figure}

\begin{assumption}[Simultaneous Diagonalization]
\label{ass:simul_diag}
Consider the rank-$r$ Kronecker decomposition $\{(A_k,B_k)\}_{k=1}^{r}$ in Assumption~\ref{ass:block-curvature}. 
Assume that $\{A_k\}_{k=1}^{r}$ and $\{B_k\}_{k=1}^{r}$ are simultaneously orthogonally diagonalized by orthogonal matrices $U \in \mathbb{R}^{d_1 \times d_1}$ and $V \in \mathbb{R}^{d_2 \times d_2}$, respectively. 
That is, for any $k=1,\ldots,r$, we have
$A_{k}
=U\Diag(a_{k}^{(1)},\ldots,a_{k}^{(d_1)}) U^\top$, and $
B_{k}=V\Diag(b_{k}^{(1)},\ldots,b_{k}^{(d_2)}) V^\top.$
\end{assumption}
This assumption states that the Kronecker factors $\{A_k\}_{k=1}^{r}$ and $\{B_k\}_{k=1}^{r}$ of the Hessian $\calH$ approximately share common orthogonal eigenbases. 
To empirically verify this assumption, we compute the simultaneous-diagonalization score
$\eta_{\mathrm{sd}}(\{X_k\}_{k=1}^r)=\max_{Q^\top Q=I}\sum_{k=1}^r \|\Diag(Q^\top X_k Q)\|_F^2/\sum_{k=1}^r \|X_k\|_F^2$
for both $\{A_k\}_{k=1}^{r}$ and $\{B_k\}_{k=1}^{r}$. The maximization over orthogonal $Q$ is approximated using the Joint Approximate Diagonalization of Eigenmatrices (JADE) Jacobi-sweep algorithm~\citep{cardoso1993blind}; see Appendix~\ref{app:simul_diag} for details.
Larger values of $\eta_{\mathrm{sd}}$ (up to a maximum of $1$) indicate that a common orthogonal basis captures more of the total energy of these matrices through their diagonal components. 
Figure~\ref{fig:joint_diag} reports the average $\eta_{\mathrm{sd}}$ scores as a bar chart: $\{A_k\}$ achieves $0.892$ and $\{B_k\}$ achieves $0.845$, both close to the maximum of $1.0$. These high scores indicate that a single shared orthogonal basis captures the vast majority of each factor family's energy, supporting Assumption~\ref{ass:simul_diag}.

Given the shared eigenbases from Assumption~\ref{ass:simul_diag}, we define paired curvatures that characterize the Hessian along each joint eigenmode. Recall that $\{A_k\}$ and $\{B_k\}$ are the Kronecker factors in the Hessian decomposition (Assumption~\ref{ass:block-curvature}), and $a_k^{(i)}$, $b_k^{(i)}$ are their respective eigenvalues under the shared bases $U$ and $V$ (Assumption~\ref{ass:simul_diag}). Let $d'=\min\{d_1,d_2\}$. For $i\in[d']$, let $w_{i}$ denote the $i$-th largest value of $\sum_{k=1}^r a_{k}^{(i^{\prime})}b_{k}^{(i^{\prime})}$ for $i^{\prime}\in[d']$. We then assume heterogeneity among the positive $w_i$.
\begin{assumption}[Curvature Heterogeneity]\label{ass:main-ordering}
Assume that exactly $q$ values in $w_i$ for $i\in[d^\prime]$ are positive and that these positive curvatures have a two-level structure: $w_i=w_{\rm H}$ for $i\in[m]$ and $w_i=w_{\rm L}$ for $i\in\{m+1,\ldots,q\}$ with $w_{\rm H}>w_{\rm L}$ and $\alpha=m/q<1/2$.
\end{assumption}
This assumption states that the positive paired curvatures of the Hessian are heterogeneous. 
We adopt the two-level structure only to simplify the calculation; the same intuition extends to more general heterogeneous spectra. 
Figure~\ref{fig:hetero} plots the positive paired curvatures $w_i$ (averaged over attention parameters) on a log scale. The bar chart reveals a strongly long-tailed distribution: the first few curvatures are on the order of $10^{-2}$, while the tail drops to $10^{-8}$, spanning over six orders of magnitude with $w_1/w_{88}\approx 2.59\times 10^6$.
Among the top-$128$ values of $|w_i|$, 88 are positive.
In addition, the condition $\alpha<1/2$ requires the proportion of high-curvature modes to be relatively small, which is consistent with the long-tailed pattern in the figure: only the first few indices carry curvatures orders of magnitude above the rest. These observations support Assumption~\ref{ass:main-ordering}.

\begin{figure}[t]
\centering
\subfigure[Simultaneous-diagonalization.\label{fig:joint_diag}]{
\includegraphics[width=0.3\textwidth]{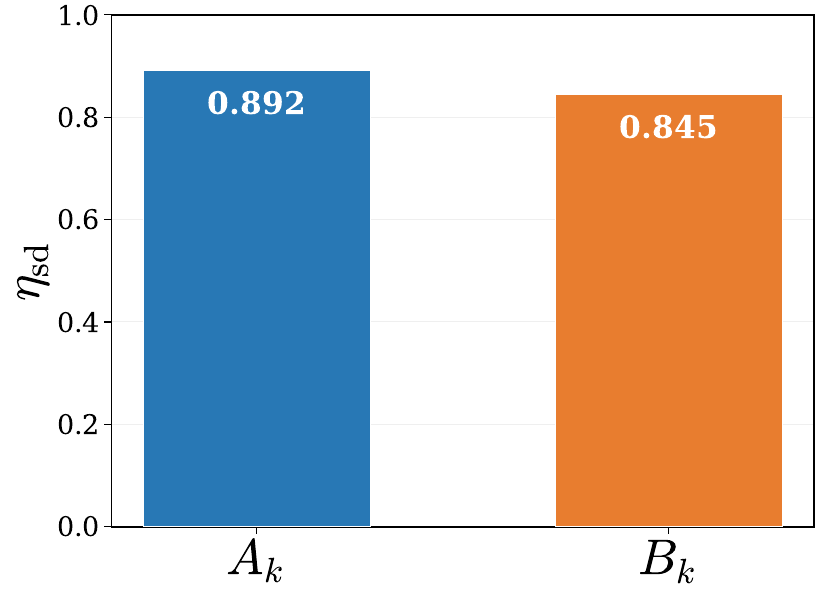}
}%
\hspace{0.003\textwidth}%
\subfigure[Positive heterogeneous curvatures.\label{fig:hetero}]{
\includegraphics[width=0.3\textwidth]{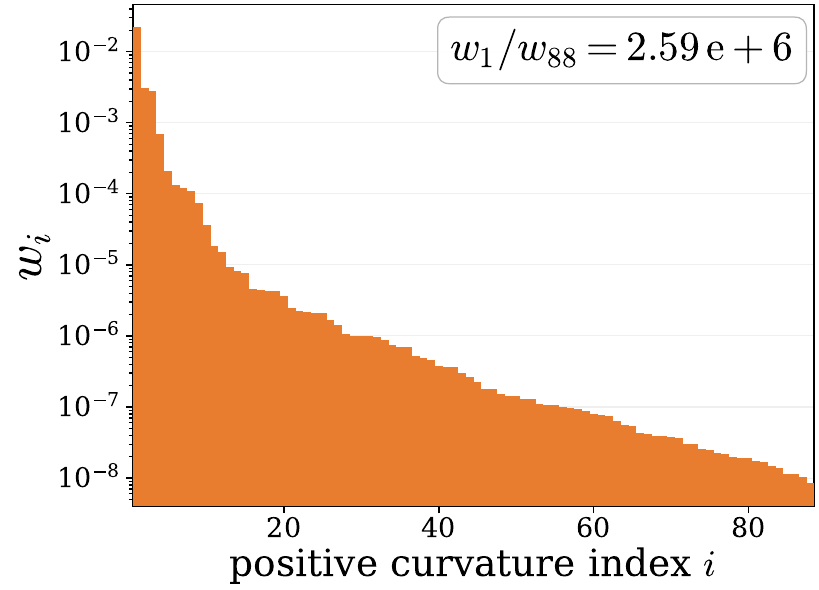}
}
\subfigure[Gradient energy ratio $\zeta(i)$.\label{fig:main-ordering-alignment}]{
\includegraphics[width=0.3\textwidth]{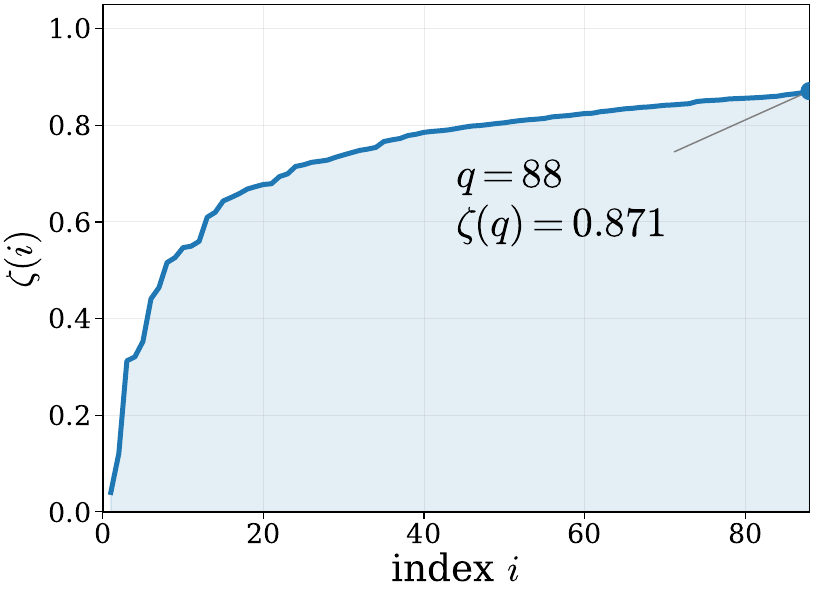}
}
\caption{Empirical support for Assumption~\ref{ass:simul_diag}-- \ref{ass:main-alignment}.
Panel (a) shows the average values of the simultaneous-diagonalization score $\eta_{\mathrm{sd}}$ for $\{A_k\}_{k=1}^r$ and $\{B_k\}_{k=1}^r$. Panel (b) shows the value of positive curvatures. Panel (c) shows the cumulative gradient energy ratio $\zeta(i)$. These results support the approximate simultaneous diagonalization of the matrices in the Hessian decomposition, as well as the alignment of gradients with the eigenvectors associated with the top eigenvalues of the Hessian spectrum.}

\label{fig:main-ordering-support}
\end{figure}

We now define the rank-one basis matrices that diagonalize the Hessian in the positive-curvature subspace. Let $u_i\in\bbR^{d_1}$ and $v_i\in\bbR^{d_2}$ denote the $i$-th columns of the shared eigenbases $U$ and $V$ from Assumption~\ref{ass:simul_diag}, respectively.
Let $\pi(i)$ denote the index of the $i$-th largest paired curvature among $\{w_{i'}=\sum_{k=1}^r a_k^{(i')}b_k^{(i')}:i'\in[d']\}$.
We define the rank-one matrix $M_i=u_{\pi(i)}v_{\pi(i)}^\top\in\bbR^{d_1\times d_2}$ for $i\in[q]$, which represents the $i$-th curvature eigenmode of the Hessian.
Since $U$ and $V$ are orthogonal, the matrices $\{M_i\}_{i=1}^{q}$ are orthonormal under the Frobenius inner product, and the Hessian acts diagonally on them: $\mathcal{H}[M_i]\approx w_i M_i$.
\begin{assumption}[Gradient Alignment]
\label{ass:main-alignment}
The gradient $G$ is in the subspace spanned by $\{M_i\}_{i=1}^{q}$, i.e.,  $G=\sum_{i=1}^q \sigma_{i} M_{i}$, where $\sigma_{i}=\langle G,M_{i}\rangle.$
In addition, the coefficients $\sigma_{i}$ have the same
two-group structure as Assumption~\ref{ass:main-ordering}: $\sigma_{i}=\sigma_{\rm H}, i\in[m]$, and $\sigma_{i}=\sigma_{\rm L}$, $i\in\{m+1,\ldots,q\}$ with $\sigma_{\rm H}>\sigma_{\rm L}$.
\end{assumption}
This assumption states that the gradient largely lies in the top-curvature subspace of the Hessian. 
Such gradient--Hessian alignment has been widely studied in deep learning~\citep{gur2018gradient,fort2019emergent}. 
To verify this assumption, we define $\calM_{i}$ as the subspace spanned by $\{M_{j}\}_{j=1}^{i}$ with $i\in[q]$ and let $\Pi_{\calM_{i}}$ denote the projection onto this subspace. 
Figure~\ref{fig:main-ordering-alignment} plots the cumulative gradient energy ratio $\zeta(i)=\|\Pi_{\mathcal M_{i}}G\|_F^2/\|G\|_F^2$ as a function of the number of included curvature directions. The curve rises steeply: by index $i\approx 30$ (roughly one-third of the $q=88$ positive-curvature directions), the cumulative energy $\zeta$ already exceeds $0.8$, and it reaches $\zeta(q)=0.871$ at the full positive-curvature subspace. This confirms that the gradient is strongly aligned with the top-curvature directions. 

Together, Assumptions~\ref{ass:block-curvature}--\ref{ass:main-alignment} reduce the local quadratic model to a low-dimensional approximation on the positive-curvature subspace, where $\mathcal H[M_i]\approx w_i M_i$ and $G\approx\sum_{i=1}^q \sigma_i M_i$. 
Here, the paired rank-one modes $M_i$ have positive and heterogeneous curvatures $w_i$, and higher-curvature modes tend to carry larger gradient energy.

We next analyze Muon's advantage in optimizing this quadratic problem.
While Section~\ref{sec:main-results} compares Muon with Adam empirically, Adam's coordinate-wise normalization does not admit a closed-form analysis within the active-mode framework above. However, on these stylized quadratic problems, Adam behaves similarly to GD and differs substantially from Muon.
Figure~\ref{fig:synth-muon-gd-signgd} confirms this: the average \ac{nds} and loss decrease for Adam and GD both behave much closer to each other than to Muon.
Therefore, to simplify the analysis, we focus theoretically on GD and Muon. Specifically, they update parameter $Y$ as follows.


\begin{figure}[t]
\centering
\subfigure[\ac{nds} ratio\label{fig:ratio-a}]{
\begin{minipage}[b]{0.32\textwidth}
\centering
\includegraphics[width=\linewidth]{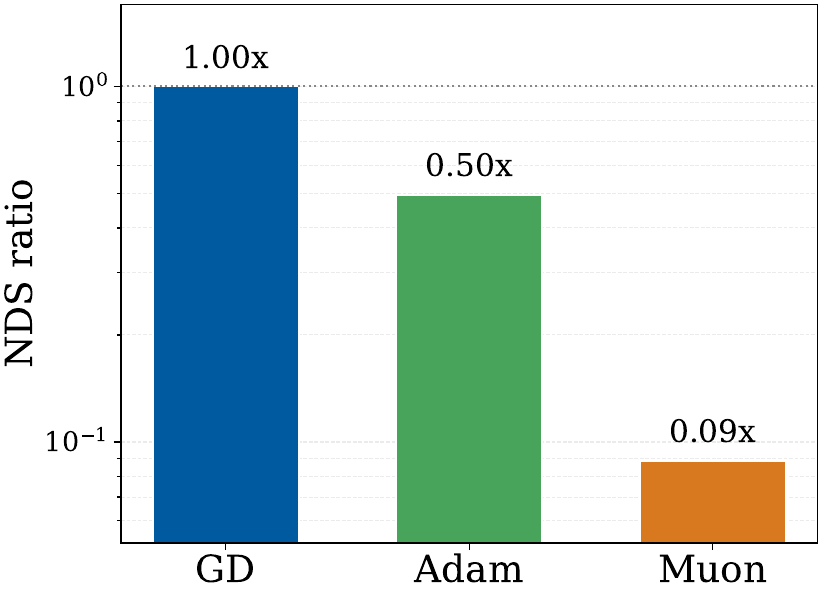}
\end{minipage}
}
\hspace{1em}
\subfigure[Loss decrease ratio\label{fig:ratio-b}]{
\begin{minipage}[b]{0.32\textwidth}
\centering
\includegraphics[width=\linewidth]{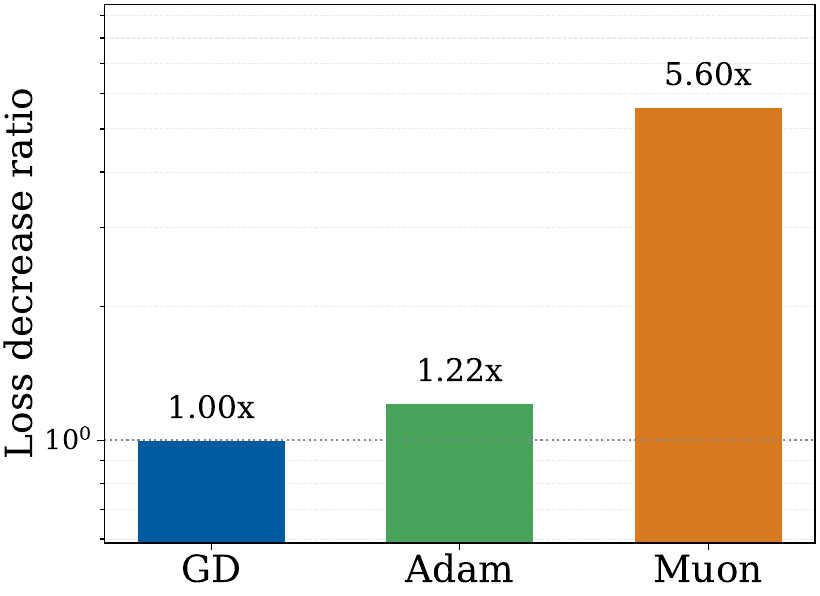}
\end{minipage}
}
\caption{Panel~(a) reports the \ac{nds} ratio, and Panel~(b) reports the loss-decrease ratio.
The results show that GD and Adam exhibit similar behavior in both \ac{nds} and loss decrease on quadratic problems satisfying Assumptions~\ref{ass:block-curvature}--\ref{ass:main-alignment}.}

\label{fig:synth-muon-gd-signgd}
\end{figure}

\begin{itemize}[leftmargin=1em]
\item  $\gd$ updates the parameter with the gradient, i.e., $Y_{t+1}^{\gd} = Y_{t}^{\gd}-\eta_t^{\gd}\nabla \calQ(Y_{t}^{\gd})=Y_{t}^{\gd}+\eta_t^{\gd}Z_t^{\gd}$.
\item With momentum set to $0$, Muon updates the parameters using the spectrally normalized gradient:
$Y_{t+1}^{\muon}=Y_{t}^{\muon}-\eta_t^{\muon}\spec(\nabla \calQ(Y_{t}^{\muon}))=Y_{t}^{\muon}+\eta_t^{\muon}Z_t^{\muon}$,
where $\spec(G)=UV^\top$ for $G=U\Sigma V^\top$, which normalizes all nonzero singular values $\Sigma$ of $G$ to one.
\end{itemize}

\subsection{Theoretical Results}\label{sec:theory}
Without loss of generality, both optimizers are initialized at $0$ and use exact line-search step sizes for a fair comparison. Specifically, for each $\opt\in\{\muon,\gd\}$, $\eta_t^{\opt}$ is set  as $\arg\max_{\eta\ge0}\{
\mathcal Q(Y_t^{\opt})-\mathcal Q(Y_t^{\opt}+\eta Z_{t}^{\opt})\}.$ Following  Eqn.~\eqref{eq:taylor_sharpness}, we define the step-wise
\ac{nds} by $\Sc_F(Z_t^{\opt})=\langle Z_t^{\opt},\mathcal H[Z_t^{\opt}]\rangle/\|Z_t^{\opt}\|_F^2$ and the finite-horizon averaged \ac{nds} by $\bar\Sc_T^{\opt}=T^{-1}\sum_{t=0}^{T-1}\Sc_F(Z_t^{\opt})$.

\begin{theorem}\label{thm:sharpness}
Let $\alpha=m/q$ denote the relative size of the high-curvature group, and let $\rho=w_{\rm H}/w_{\rm L}>1$ denote the curvature ratio. 
Under Assumptions~\ref{ass:block-curvature}--\ref{ass:main-alignment}, the following result holds.
    \begin{itemize}[leftmargin=1em]
        \item \textbf{Smaller \ac{nds} of Muon.} 
        For every finite horizon \(T\ge1\), Muon has a smaller finite-horizon averaged \ac{nds} than that of GD, i.e., $
        \bar \Sc_{T}^{\muon}<\bar \Sc_{T}^{\gd}.$
        \item \textbf{Larger Loss Decrease of Muon.} If $\rho+1>1/\alpha>1+\sigma_{\rm H}/\sigma_{\rm L}$ , for every finite horizon $T\ge 1$, Muon achieves a lower loss than GD, i.e.,  $\calQ(Y_T^{\muon})< \calQ(Y_T^{\gd})$.
    \end{itemize}
\end{theorem}
Theorem~\ref{thm:sharpness} gives two results.
First, it shows that Muon has a smaller average \ac{nds} than that of GD for any finite horizon, providing theoretical support for Observation~2.
Second, under sufficient curvature \kw{heterogeneity}---specifically, when the curvature ratio satisfies $\rho>1/\alpha-1$ and the high-curvature group is small enough that $\alpha<\sigma_{\rm L}/(\sigma_{\rm L}+\sigma_{\rm H})$---Muon achieves a larger cumulative loss decrease than GD, supporting Observation~1.
The condition $\rho>1/\alpha-1$ requires the largest curvature values to be much larger than the others, consistent with Hessian eigenvalue outliers observed in prior work~\citep{sagun2017empirical,fort2019emergent}.
The condition $\alpha<1/(1+\sigma_{\rm H}/\sigma_{\rm L})$ requires the high-curvature group to occupy a sufficiently small fraction of the positive-curvature subspace, consistent with the small proportion of outlier curvature directions observed in practice~\citep{sagun2017empirical}.

The key mechanism behind both results is that Muon's spectral normalization \kw{equalizes the update amplitude} across all orthogonal curvature eigenmodes, distributing energy evenly between high- and low-curvature directions. In contrast, GD's update is proportional to the gradient, which concentrates more energy on high-curvature directions (by Assumption~\ref{ass:main-alignment}). This concentration causes GD to incur larger directional sharpness and, when curvature heterogeneity is strong enough, a larger curvature penalty that offsets its first-order gain.

We also highlight that our curvature perspective helps explain why some Muon variants improve training efficiency. The recent work of \citet{zhu2026accelerating} enhances Muon by increasing its update component toward flat Hessian directions, leading to further efficiency gains. This finding is supported by our analysis: shifting the update toward flatter directions should reduce \ac{nds}, thereby lowering the curvature penalty and potentially increasing the one-step loss decrease.

\subsection{Proof Sketch}
\label{subsec:proof_sketch}

We outline the main idea behind Theorem~\ref{thm:sharpness}; the full algebra is deferred to Appendix~\ref{app:two-group-proof}. 
Under Assumptions~\ref{ass:block-curvature}--\ref{ass:main-alignment}, the local quadratic model diagonalizes on the span of the orthonormal rank-one modes $\{M_i\}_{i=1}^q$. 
In this basis, we have that $
    G=\sum_{i=1}^q \sigma_i M_i$ with $
    \mathcal H[M_i]=w_iM_i,$
and hence, for any $Y=\sum_i y_iM_i$, the quadratic model decomposes into scalar components:
\begin{equation*}
    \mathcal Q(Y)
    =
    \mathcal L(W_0)-\sum_{i=1}^q\sigma_i y_i
    +\frac12\sum_{i=1}^q w_i y_i^2 .
\end{equation*}
Here, $w_i$ is the curvature of mode $M_i$, $\sigma_i$ is the gradient coefficient on that mode, and $y_i$ is the coordinate of the update $Y$. The two-group structure means that the first $m$ modes have curvature $w_{\rm H}$ and gradient coefficient $\sigma_{\rm H}$, while the remaining $q-m$ modes have curvature $w_{\rm L}$ and coefficient $\sigma_{\rm L}$. We write $\alpha=m/q$ for the fraction of high-curvature modes, with $w_{\rm H}>w_{\rm L}$ and $\sigma_{\rm H}>\sigma_{\rm L}$.

Consequently, the dynamics are fully characterized by the residual gradient coefficients $r_{i,t}^{\opt}=\sigma_i-w_i y_{i,t}^{\opt}$, where $\opt\in\{\muon,\gd\}$ and $Y_t^{\opt}=\sum_i y_{i,t}^{\opt}M_i$. In this coordinate system, the distinction between Muon and GD is transparent. Since the matrices $M_i$ are paired singular directions, spectral normalization maps the residual $\sum_i r_{i,t}^{\muon}M_i$ to
\begin{equation*}
    Z_t^{\muon}=\sum_{i=1}^q \sgn(r_{i,t}^{\muon})M_i,
\end{equation*}
so Muon assigns the same amplitude to every active mode. GD instead uses
\begin{equation*}
    Z_t^{\gd}=\sum_{i=1}^q r_{i,t}^{\gd}M_i,
\end{equation*}
so its update energy is proportional to the current residual energy. This is the central mechanism: Muon removes magnitude imbalance across modes, while GD inherits it.

\paragraph{\ac{nds} comparison.} Muon's equal-amplitude update induces a fixed curvature average,
\begin{equation*}
\Sc_F(Z_t^{\muon})
=
\alpha w_{\rm H}+(1-\alpha)w_{\rm L},
\end{equation*}
where $\alpha$ denotes the fraction of update directions lying in the high-curvature subspace under Assumption~\ref{ass:main-ordering}. GD's \ac{nds} is instead a residual-energy-weighted average,
\begin{equation*}
    \Sc_F(Z_t^{\gd})
    =
    P_t^{\gd}w_{\rm H}+(1-P_t^{\gd})w_{\rm L},\qquad
    P_t^{\gd}
    =
    \frac{\sum_{i=1}^m(r_{i,t}^{\gd})^2}
    {\sum_{i=1}^q(r_{i,t}^{\gd})^2}.
\end{equation*}
Here $P_t^{\gd}$ is the fraction of GD's residual energy lying in the high-curvature group at step $t$. Because $\sigma_{\rm H}>\sigma_{\rm L}$, GD starts with more residual energy in the high-curvature group than the group-size fraction alone would suggest:
\begin{equation*}
    P_0^{\gd}=p
    =
    \frac{m\sigma_{\rm H}^2}
    {m\sigma_{\rm H}^2+(q-m)\sigma_{\rm L}^2}
    >\alpha .
\end{equation*}
The update of GD then makes it overshoot the high-curvature group and undershoot the low-curvature group, which yields the recursion $P_{t+1}^{\gd}=1-P_t^{\gd}$ (Proposition~\ref{prop:gd_sharp}). Hence the time-averaged high-curvature weight of GD is always larger than $\alpha$, whereas Muon always assigns exactly an $\alpha$ fraction of its update energy to the high-curvature group. Since $w_{\rm H}>w_{\rm L}$, this gives $\bar\Sc_T^{\muon}<\bar\Sc_T^{\gd}$ for every finite horizon $T$, where $\bar\Sc_T^{\opt}=T^{-1}\sum_{t=0}^{T-1}\Sc_F(Z_t^{\opt})$.

\paragraph{Loss comparison.} The same residual viewpoint gives the terminal gap
\begin{equation*}
    \Phi_t^{\opt}
    =
    \mathcal Q(Y_t^{\opt})-\mathcal Q(Y^\star)
    =
    \frac12\sum_{i=1}^q\frac{(r_{i,t}^{\opt})^2}{w_i}.
\end{equation*}
Here $Y^\star$ is the minimizer of $\mathcal Q$ in the active mode subspace, and $\Phi_t^{\opt}$ is the suboptimality of optimizer $\opt\in\{\gd,\muon\}$ after $t$ steps. By the two-group symmetry, all high-curvature modes share a common residual $r_{\rm H,t}^{\muon}$ and all low-curvature modes share a common residual $r_{\rm L,t}^{\muon}$. Muon's equal-amplitude steps contract the two scaled residuals $|r_{\rm H,t}^{\muon}|/w_{\rm H}$ and $|r_{\rm L,t}^{\muon}|/w_{\rm L}$ toward one another at rate
\begin{equation*}
    \Gamma
    =
    \frac{|mw_{\rm H}-(q-m)w_{\rm L}|}{mw_{\rm H}+(q-m)w_{\rm L}} .
\end{equation*}
This form follows from Muon's update rule. Muon pulls both curvature groups toward the same weighted average, and the remaining gap is the old gap multiplied by the normalized imbalance between the total curvature weights $mw_{\rm H}$ and $(q-m)w_{\rm L}$, which gives $\Gamma$. GD, by contrast, alternates its residual energy between the two curvature groups; its contraction is governed by the two-step factor $R=C(P_0^{\gd})C(1-P_0^{\gd})$, where $C(\cdot)$ is defined as
\begin{equation*}
    C(x)
    =
    \frac{(w_{\rm H}-w_{\rm L})^2x(1-x)}
    {(w_{\rm L}+(w_{\rm H}-w_{\rm L})x)^2}.
\end{equation*}
To see this expression, suppose GD's current high-curvature residual-energy share is $x$. Exact line search along the residual sets
\begin{equation*}
    \eta_t^{\gd}
    =
    \frac{1}{w_{\rm L}+(w_{\rm H}-w_{\rm L})x},
\end{equation*}
the reciprocal of the residual-weighted average curvature. The high-curvature residuals are multiplied by $1-\eta_t^{\gd}w_{\rm H}$, while the low-curvature residuals are multiplied by $1-\eta_t^{\gd}w_{\rm L}$. Squaring these two factors and averaging them with weights $x$ and $1-x$ gives exactly $C(x)$. Hence $C(x)$ is the one-step GD residual-energy contraction at energy share $x$, and $R$ is GD's two-step contraction factor over the alternating pair $P_0^{\gd}$ and $1-P_0^{\gd}$. With these analyses of Muon and GD, we can show that
\begin{equation*}
    \Phi_T^{\muon}=\Phi_1^{\muon}\Gamma^{2(T-1)},
    \qquad
    \Phi_T^{\gd}=\Phi_1^{\gd}R^{(T-1)/2}.
\end{equation*}
Under the condition $\rho+1>1/\alpha>1+\sigma_{\rm H}/\sigma_{\rm L}$, where $\rho=w_{\rm H}/w_{\rm L}$ is the curvature ratio, Muon has both a smaller first-step gap, $\Phi_1^{\muon}<\Phi_1^{\gd}$, and a faster subsequent contraction, $\Gamma^2<\sqrt R$ (Propositions~\ref{prop:onestep_gap} and~\ref{prop:muon_contract}). Combining these two inequalities yields $\mathcal Q(Y_T^{\muon})<\mathcal Q(Y_T^{\gd})$ for all $T\ge1$.

Overall, the proof formalizes a simple intuition: when the gradient is biased toward a small set of sharp modes, GD follows that bias and repeatedly spends too much update energy in high-curvature directions. Muon's spectral normalization spreads the update evenly across the active singular modes, lowering directional sharpness and making more balanced progress across sharp and flat directions.

\section{Conclusion}
\label{sec:conclusion}

Our work takes a first step toward understanding Muon's superiority over Adam from a curvature perspective. 
We show that Muon's lower \ac{nds} reduces its curvature penalty, thereby yielding a larger one-step loss decrease than Adam. 
We further show that this \ac{nds} advantage is shaped by data imbalance and layer-wise interactions in the model. 
We also provide a quadratic-model analysis that theoretically establishes Muon's advantages in \ac{nds} and one-step loss decrease. 
A limitation of our work is that we focus on causal \acp{llm}; we leave the verification of these insights on other model classes, such as diffusion models, to future work.

\newpage
\bibliographystyle{ims}
\bibliography{ref}

\newpage 

\appendix

\section{Experimental Details}
\label{app:exp-details}

\subsection{FineWeb Main-Text Experiments}
\label{app:fineweb-experiments}

We use a 124M-parameter NanoGPT model with 12 Transformer layers, 12 attention heads, and hidden dimension of 768. 
The vocabulary size is 50{,}257 under the GPT-2 tokenizer. 
Training is conducted on FineWeb-10B with sequence length 1024~\citep{penedo2024fineweb}. For Muon, we apply Muon to all matrix parameters except the token embeddings and language-model head, using momentum $\mu=0.95$ and no weight decay. 
The momentum coefficient is linearly warmed up from $0.85$ to $0.95$ over the first 300 steps, and Newton--Schulz orthogonalization uses 5 iterations. 
The embedding layer, \texttt{lm\_head}, and all scalar or one-dimensional parameters are optimized by Adam with $\beta_1=0.8$ and $\beta_2=0.95$. 
For the Adam baseline, all parameters, including the attention and MLP matrices, are optimized by Adam with $\beta_1=0.8$ and $\beta_2=0.95$. 
For both Adam and Muon, learning rates are selected by grid search from $\{1,2,5\}\times\{10^{-1},10^{-2},10^{-3},10^{-4}\}$. 
We use Modded-NanoGPT as our code base~\citep{modded_nanogpt_2024}.

Directional sharpness is computed every 500 steps through Hessian--vector products. We conduct experiments on 4 A100 with 80 GB memory.

\subsection{Zipf-PCFG Dataset Construction and Experiment details in Section~\ref{sec:imbalance-results}}
\label{app:zipf-pcfg-dataset}

\subsubsection{Dataset Construction}

The dataset used in Section~\ref{sec:imbalance-results} is a synthetic Zipf-PCFG corpus
with latent topic structure and explicit grammatical constraints.

\paragraph{Vocabulary.} The vocabulary $\calV$ contains $4{,}411$ tokens
partitioned into $C=20$ grammatical classes (the $4{,}412^{\text{nd}}$
slot is the EOS token) so that $\calV=\cup_{c=1}^C\calV_c$. 

Class sizes are fixed at
$|N_{\text{anim}}|=700$, $|N_{\text{inanim}}|=1200$,
$|V_{\text{intrans}}|=380$, $|V_{\text{trans}}|=570$,
$|V_{\text{clause}}|=280$, $|V_{\text{dative}}|=190$, $|V_{\text{modal}}|=10$,
$|V_{\text{neg}}|=5$, $|\text{Adj}_{\text{grad}}|=380$,
$|\text{Adj}_{\text{nongrad}}|=330$, $|\text{Adv}_{\text{deg}}|=20$,
$|\text{Adv}_{\text{manner}}|=280$, $|P|=25$, $|\text{Det}_{\text{def}}|=
|\text{Det}_{\text{indef}}|=|\text{Comp}|=|\text{Coord}|=|\text{Punc}|=5$,
$|\text{Num}_{\text{sg}}|=|\text{Num}_{\text{pl}}|=8$.
Each class has its own internal Zipf rank.

\paragraph{Latent topics.} We instantiate $K=30$ latent topics. For each
topic $k$ and each class $c$ we draw a per-topic affinity
$\phi_{k,c} \in \Delta^{|c|-1}$ by normalising
$\mathrm{Gamma}(\alpha_\phi,1)$ with $\alpha_\phi=0.3$, so that affinities
are sparse over class tokens. The topic prior is
$\pi \sim \mathrm{Dirichlet}(\alpha_\pi=1.0)$. Documents drift through
topics under a Markov chain with transition matrix
$T_{kk} = 0.85$, $T_{k\ell} = 0.15/(K-1)$ for $\ell\ne k$, providing
strong topic persistence within each document.


\paragraph{Token emission.}
Within token class $c$ and topic $k$, each token $j\in\calV_c$ has a base topic-specific probability $\phi_{k,c}(j)$. We rank the tokens in $\calV_c$ in decreasing order of $\phi_{k,c}(j)$ and denote the rank of token $j$ by $r(j,c,k)$. Given imbalance level $s$, tokens are sampled from the reweighted distribution
\begin{equation*}
    P_s(j\mid c,k)=\frac{r(j,c,k)^{-s}\phi_{k,c}(j)}{\sum_{j'\in\calV_c} r(j',c,k)^{-s}\phi_{k,c}(j')}, \qquad j\in\calV_c .
\end{equation*}
Thus, when $s=0$, token emission is governed only by the topic-specific base distribution $\phi_{k,c}$; increasing $s$ further concentrates probability mass on higher-ranked tokens and yields heavier-tailed class-conditional token
distributions.

\paragraph{Grammatical structure.} We sample sentences from a PCFG with start symbol $\bar S$. Following
standard linguistic notation, $NP$ denotes a noun phrase, $VP$ a verb
phrase, $CP$ a complementiser phrase (clausal complement), $CP_{\text{fronted}}$
a fronted clausal complement, $\text{RelAdv}$ a relative-adverb clause,
and $\text{Coord}$ a coordinator. The seven productions are
$\bar S \to NP\, VP$ (probability $0.55$), $NP\, VP\, CP$ ($0.12$),
$\bar S\, \text{Coord}\, \bar S$ ($0.08$), $CP_{\text{fronted}}\, NP\, VP$
($0.08$), $NP\, \text{Modal-VP}$ ($0.07$), $NP\, \text{Neg-VP}$ ($0.05$),
and $NP\, VP\, \text{RelAdv}$ ($0.05$). Each $VP$ is generated under one of
four verb subcategorisation frames (intransitive, transitive, clausal,
dative) and respects: (i)~explicit number agreement between subject and
verb (with each noun assigned a fixed singular/plural feature, $\sim
40\%$ plural); (ii)~selectional restrictions, with animate nouns
preferred as agents (preference $5\times$) and inanimate nouns preferred
as patients (preference $3\times$); (iii)~role coherence (the agent and
patient of a clause are required to differ); and (iv)~negation and
modality scope tracking. Recursion is bounded at depth $14$.

\paragraph{Document and tokenisation.} Each sample is a multi-sentence
document obtained by repeatedly sampling sentences while drifting
through topics under the HMM above. Documents are flattened to a single
token stream, mapped to the integer vocabulary, and concatenated into
binary training shards. The training and validation Zipf exponents can
be set independently. The main comparison fixes the validation distribution at $s=0$, where the
rank-based Zipf bias is removed, and token emission is governed only by the
topic-specific base distribution $\phi_{k,c}$. Then, we vary the
training exponent $s\in\{0,0.5,1\}$, isolating the effect of training imbalance on optimizer behaviour.

\subsubsection{Experiment Setup} Different from our FineWeb experiments, considering the reduced task difficulty, we use a 9M-parameter NanoGPT model with 4 Transformer layers, 4 attention heads, and model
dimension 256. The vocabulary size is $4{,}412$ under the synthetic Zipf--PCFG tokenizer. Training uses the Zipf--PCFG corpus with sequence length $1{,}024$ and gradient
accumulation $8$, giving an effective batch size of $8{,}192$ tokens per
step. Validation uses $2{,}457{,}600$ tokens with sequence length $1{,}024$.
Total training is $10{,}000$ optimization steps. We compare two optimizers. For Muon, we apply Muon to all attention matrices ($W_Q$, $W_K$, $W_V$, $W_O$) and all
MLP matrices ($W_{\text{in}}$, $W_{\text{out}}$), with learning rate
$5\times 10^{-3}$, momentum $\mu=0.95$, and Newton--Schulz orthogonalization $5$ iterations.
The embedding, \texttt{lm\_head}, and all scalar (1D) parameters are
optimized by Adam with learning rate $1\times 10^{-2}$, $\beta_1=0.8$,
$\beta_2=0.95$.
For Adam, all parameters, including the attention and
MLP matrices, are optimized by Adam with learning rate
$1\times 10^{-2}$, $\beta_1=0.8$, $\beta_2=0.95$.

\section{Additional Experimental Results to Section~\ref{sec:sharpness-comparison}}\label{app:sharpness}

This section complements the main \ac{nds} comparison results in Section~\ref{sec:sharpness-comparison} by showing the sharpness comparison over the training steps and the corresponding ratio of curvature penalty, \ac{nds} and update norm over the training steps. This result follows the same conclusion as Observation~2.

\begin{figure}[h]
\centering
\subfigure[Sharpness comparison.\label{fig:sharpness-step}]{
\includegraphics[width=0.45\textwidth]{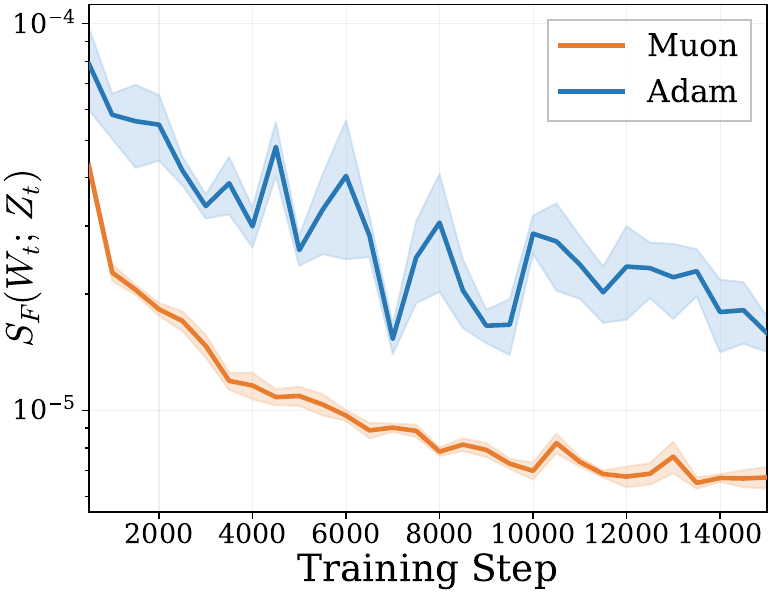}
}
\hspace{1em}
\subfigure[Adam/Muon ratios of curvature penalty, NDS, update scale\label{fig:ratio-step}]{
\includegraphics[width=0.45\textwidth]{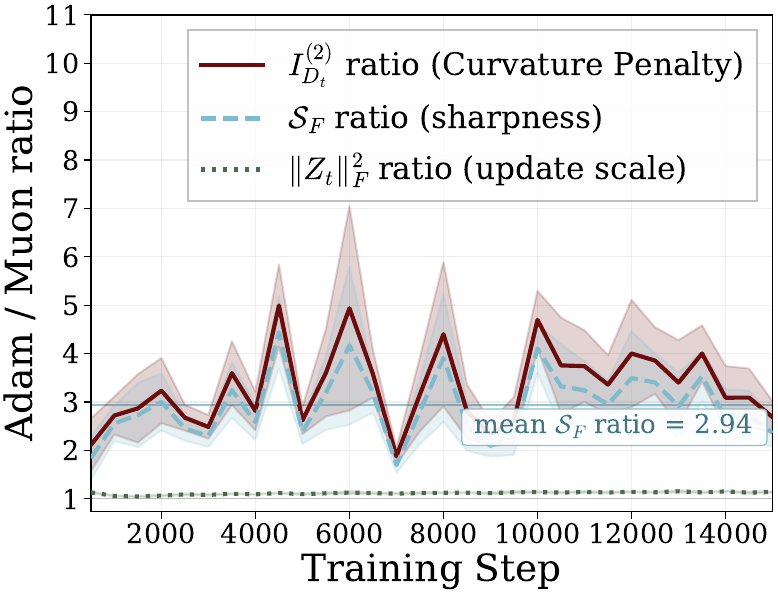}
}
\caption{\ac{nds} and the corresponding ratio comparison along the training steps of Muon and Adam. Panel (a): sharpness comparison. Panel (b): Adam/Muon ratios of the curvature penalty, the \ac{nds}, and the squared Frobenius update norm.}
\label{fig:sharpness-frob}
\end{figure}

Figure~\ref{fig:sharpness-step} shows that Muon maintains consistently smaller NDS than Adam throughout the entire training trajectory, with the gap persisting from early to late training. Figure~\ref{fig:ratio-step} reports the Adam-to-Muon ratios of the curvature penalty, NDS, and squared update norm over training steps. The update norm ratio remains close to 1, confirming that the two optimizers have comparable update norms at every step. In contrast, the NDS ratio stays well above 1, with a mean value of 2.94 over the trajectory. The curvature-penalty ratio closely tracks the NDS ratio, reinforcing Observation~2 that Muon's smaller curvature penalty is driven by lower NDS rather than smaller update norms.

We note that the mean NDS ratio at matched training steps (2.94) is larger than the ratio at matched validation loss reported in Section~\ref{sec:sharpness-comparison} (1.76). This is because Muon reaches a lower validation loss than Adam at the same training step, aligning by step effectively compares Muon at a more advanced stage of optimization against Adam at a less advanced stage, which amplifies the apparent gap. The matched-validation-loss comparison in the main text provides an assessment of the per-step curvature difference, while the step-aligned comparison here confirms that the qualitative conclusion of Observation~2 is robust to the choice of alignment.

\section{Layerwise NDS contribution}
\label{app:layer-localization}

To complement the sharpness comparison in Observation~4, Figure~\ref{fig:layer-localization} decomposes the directional sharpness into per-layer contributions
$\Sc_F^{(\ell,\opt)}=\langle Z_\ell^{\opt},\mathcal H_{\ell\ell}[Z_\ell^{\opt}]\rangle/\|Z^{\opt}\|_F^2$ for the 12 Transformer layers and optimizer $\opt\in\{\muon,Adam\}$. Panel~(a) shows that the two optimizers do not accumulate within-layer curvature uniformly across depth. Panel~(b) reports the per-layer share of the total Adam--Muon within-layer gap, $\Delta \Sc_F^{(\ell)}/\sum_{\ell=1}^{12} \Delta \Sc_F^{(\ell)}$, where $\Delta \Sc_F^{(\ell)}=\Sc_F^{(\ell,Adam)}-\Sc_F^{(\ell,\muon)}$.

\begin{figure}[h]
\centering
\includegraphics[width=0.9\textwidth]{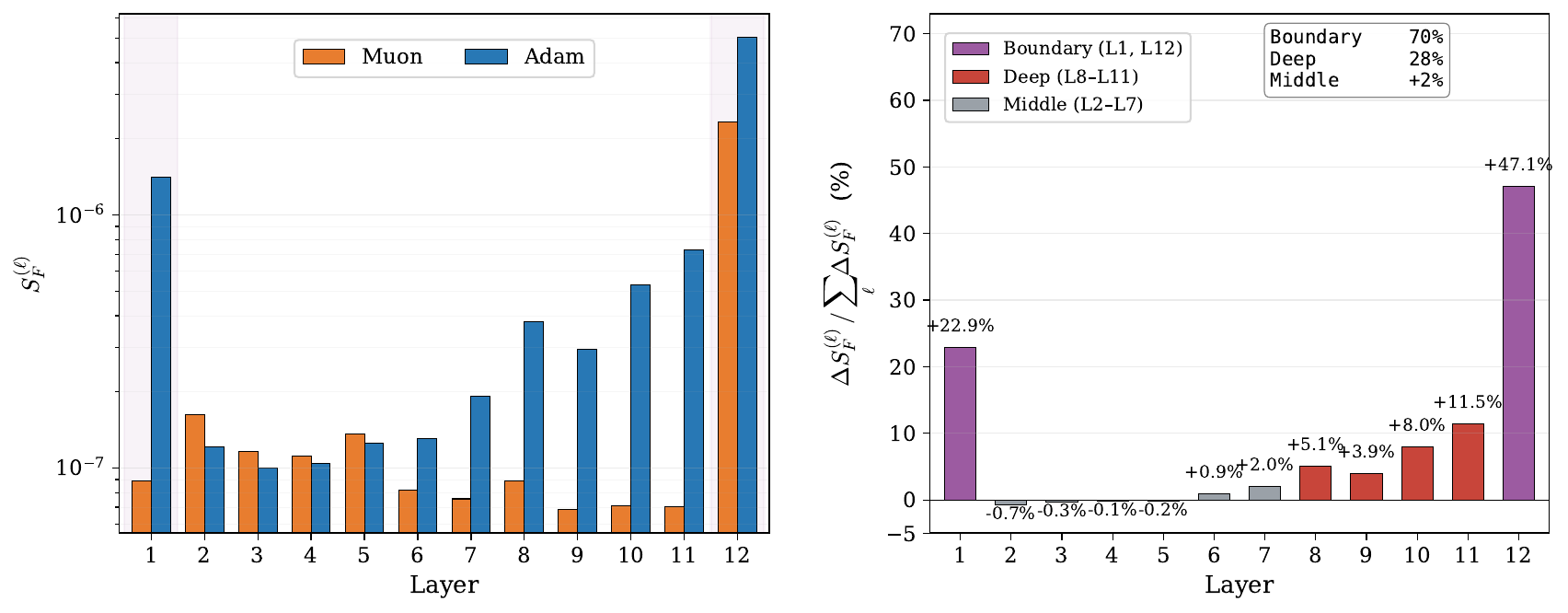}
\caption{Layerwise localization of the Adam--Muon within-layer sharpness gap across the 12 Transformer layers. Panel~(a) plots the per-layer \ac{nds} contribution $\Sc_F^{(\ell)}$ for the updates of the two optimizers. Panel~(b) reports the share ratio $\Delta \Sc_F^{(\ell)}/\sum_\ell \Delta \Sc_F^{(\ell)}$.
}
\label{fig:layer-localization}
\end{figure}

From Figure~\ref{fig:layer-localization}, we can see that the gap is strongly depth-localized. Specifically, roughly $70\%$ of the within-layer \ac{nds} gap between Muon and Adam comes from the two boundary layers L1 and L12, and about $28\%$ from the deep layers L8--L11, and only about $2\%$ from the middle layers L2--L7. The concentration of the gap at L1 and L12 is consistent with the fact that these boundary layers interact most directly with the token embedding and the output logits, where the data distribution has the most immediate influence on the local curvature.

\section{Empirical Verification Details for Section~\ref{sec:case-study-kronecker}}
\label{app:verify-relaxed-alignment}

This appendix describes the experimental setup, data source and
computational pipeline behind each verification figure in
Section~\ref{sec:case-study-kronecker}.
All experiments use the last checkpoint trained on the Zipf-PCFG dataset shown in Appendix~\ref{app:exp-details} under the Muon optimizer. Block Hessians are computed for the four attention matrices $W_Q, W_K, W_V, W_O \in \mathbb{R}^{256\times 256}$ with each dense Hessian
$\mat(\mathcal H) \in \mathbb{R}^{256^2 \times 256^2} = \mathbb{R}^{65536\times 65536}$.

\subsection{Effective low-Kronecker-rank block Hessian (Assumption~\ref{ass:block-curvature})}
\label{app:kron}

Let $\mat(\mathcal H) \in \mathbb{R}^{d_1 d_2 \times d_1 d_2}$ be the dense Hessian for one
attention block, reshaped under the Van~Loan rearrangement~\citep{van1993approximation} into a matrix
$R(\mat(\mathcal H)) \in \mathbb{R}^{d_1^2 \times d_2^2}$ for which an SVD
$R(\mat(\mathcal H)) = \sum_k s_k\, \mathrm{vec}(A_k)\, \mathrm{vec}(B_k)^\top$
yields the optimal rank-$r$ Kronecker approximation
$H_r = \sum_{k=1}^{r} B_k^T \otimes A_k$ in Frobenius norm. We compute $R(\mat(\mathcal H))$
with a randomised SVD on the $65536\times 65536$ rearranged matrix and
report the Frobenius energy ratio
\begin{equation*}
    \xi(r) = \|H_r\|_F^2 / \|\mat(\mathcal H)\|_F^2 .
\end{equation*}
Figure~\ref{fig:kronecker-assumption} plots
$\xi(r)$ for $r = 1,\dots,30$ on each of the four attention matrices.
We observe that $\xi(4) = 0.75, 0.95, 0.87, 0.71$ for $W_Q, W_K, W_V, W_O$
respectively. Panel~(b) of
Figure~\ref{fig:kronecker-assumption} visualises the residual heatmap $\mat(\mathcal H) - H_4$ on $W_V$, confirming
that the residual is small and approximately diagonal, consistent with
Assumption~\ref{ass:block-curvature}. The dense Hessian is a $65{,}536 \times 65{,}536$ matrix, far too many
to plot directly. We therefore visualise its $i$–$i'$ \emph{row-side
structure} by mean-pooling over the column indices, $P[i, i'] = \frac{1}{d_2^2} \sum_{j, j'=1}^{d_2}\big| H[(i, j), (i', j')] \big|\in\mathbb{R}^{d_1\times d_1} = \mathbb{R}^{256\times 256}$, which preserves the row-side coupling captured by the $A$-side Kronecker factors $\{A_k\}$ in Assumption~\ref{ass:block-curvature}. The same mean-pool is applied to $H_4$ and $\mat(\mathcal H) - H_4$, giving the three heatmaps. Figures~\ref{fig:wq_kron}--\ref{fig:wo_kron}
present the analogous heatmaps for $W_Q$, $W_K$, and $W_O$,
all of which exhibit similarly small residuals under the rank-4
Kronecker approximation.

\begin{figure}[h]
    \centering
    \includegraphics[width=0.9\linewidth]{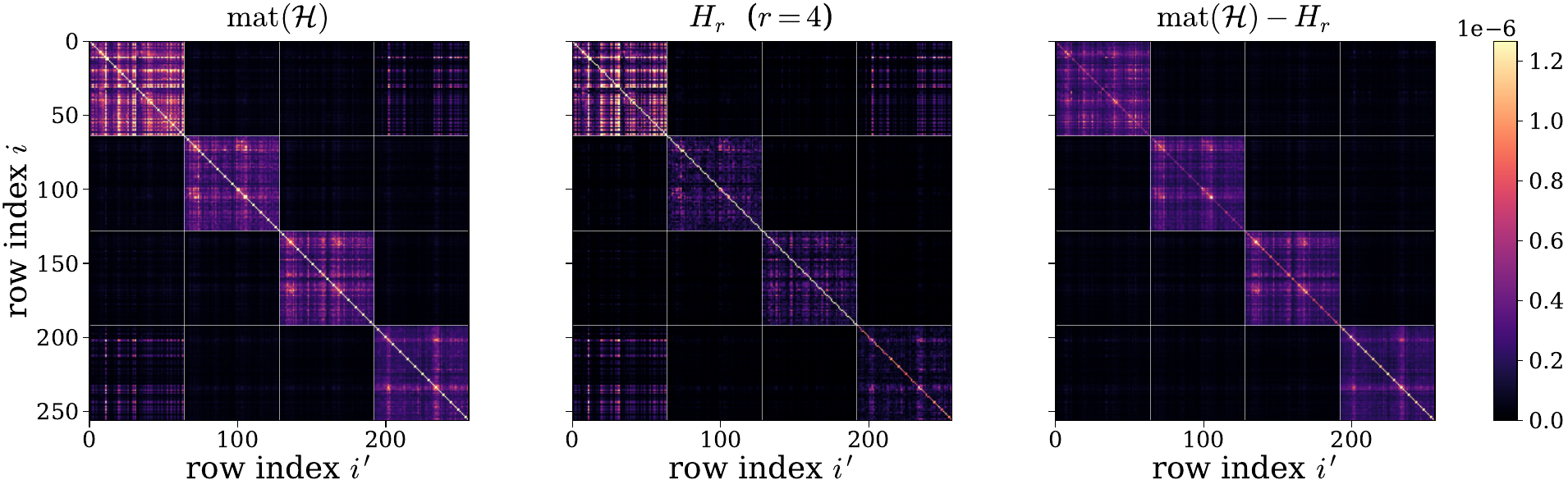}
    \caption{$W_Q$ Hessian, rank-$4$ Kronecker approximation, and residual.}
    \label{fig:wq_kron}
\end{figure}

\begin{figure}[h]
    \centering
    \includegraphics[width=0.9\linewidth]{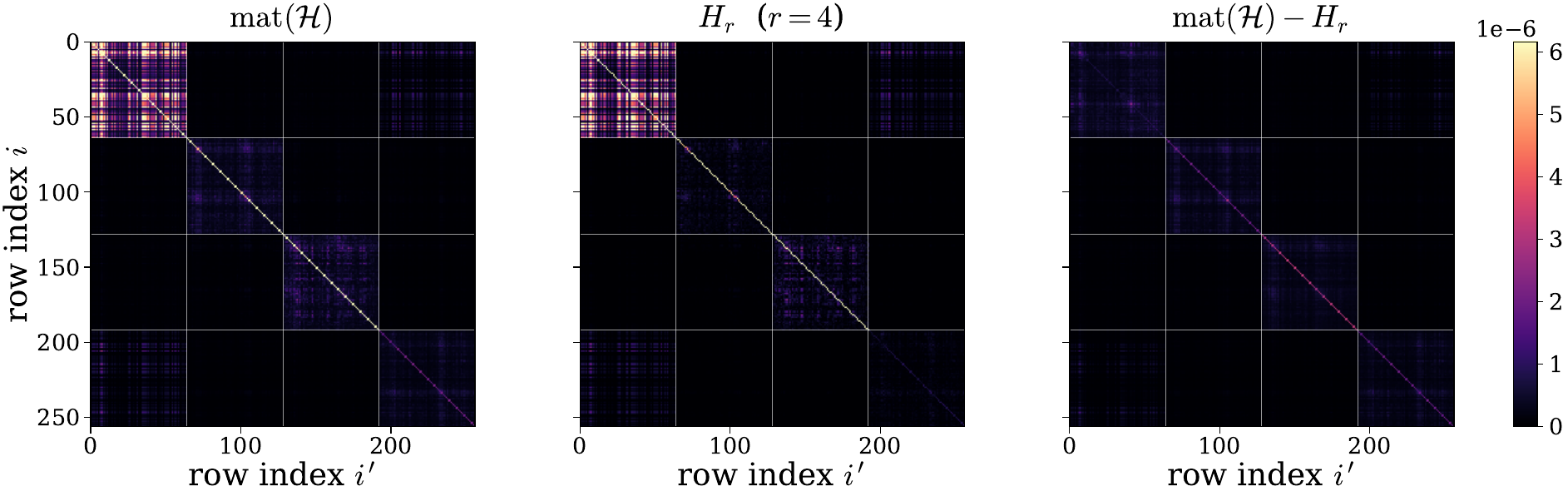}
    \caption{$W_K$ Hessian, rank-$4$ Kronecker approximation, and residual.}
    \label{fig:wk_kron}
\end{figure}

\begin{figure}[h]
    \centering
    \includegraphics[width=0.9\linewidth]{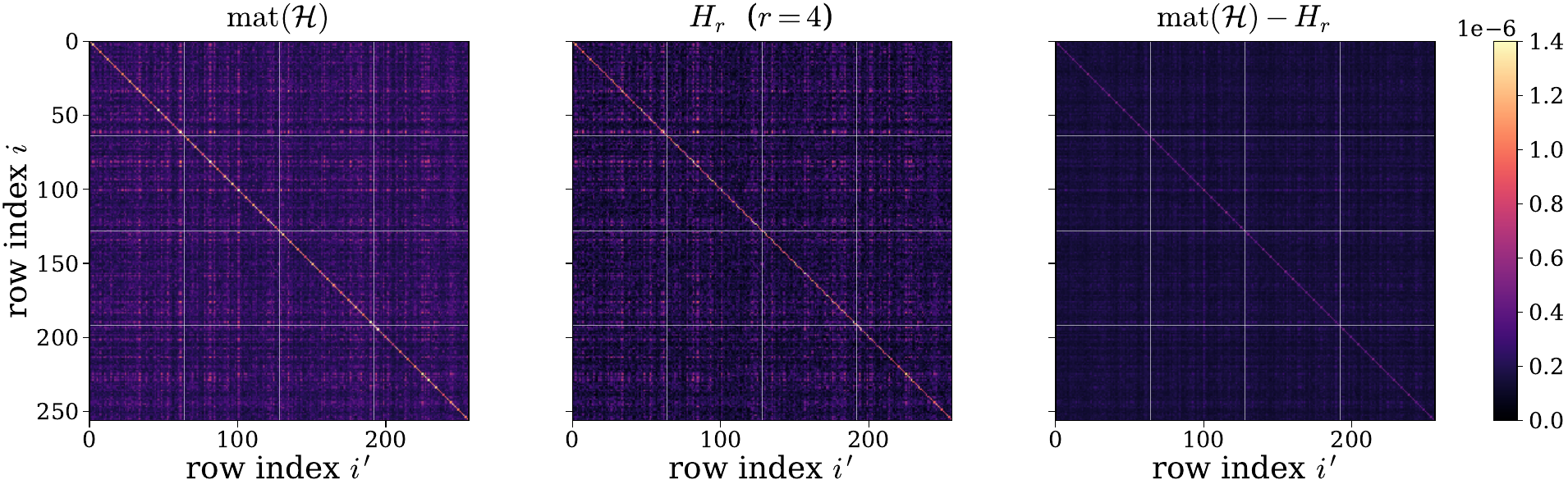}
    \caption{$W_O$ Hessian, rank-$4$ Kronecker approximation, and residual.}
    \label{fig:wo_kron}
\end{figure}

\subsection{Simultaneous Diagonalization (Assumption~\ref{ass:simul_diag})}\label{app:simul_diag}

We verify this assumption on the attention block, using Kronecker rank $r=4$. This choice balances approximation quality and computational tractability where at $r=4$, the Kronecker approximation already captures  $\xi(4) \geq 0.87$ of the Frobenius energy for three of the four attention matrices (Appendix~\ref{app:kron}), while keeping the number of matrix pairs small enough for reliable JADE joint diagonalization. For a matrix family
$\{X_k\}_{k=1}^{r}$ we compute
\begin{equation*}
\eta_{\mathrm{sd}}(\{X_k\}_{k=1}^r)
\;=\; \max_{Q^\top Q = I}
\frac{\sum_{k=1}^{r} \sum_{i} (e_i^\top Q^\top X_k Q\, e_i)^2}
     {\sum_{k=1}^{r} \|X_k\|_F^2 } .
\end{equation*}

The maximisation is approximated with the JADE Jacobi-sweep algorithm
\citep{cardoso1993blind,cardoso1996jacobi} run for 20 sweeps.
The JADE algorithm estimates an orthogonal demixing matrix by solving a joint approximate diagonalization problem. Given a collection of symmetric matrices $
\mathcal{C}=\{C_1,\ldots,C_K\},C_k\in\mathbb{R}^{d\times d},$
the goal is to find an orthogonal matrix \(V\in\mathbb{R}^{d\times d}\) such that the transformed matrices $ V^\top C_k V$, for $k=1,\ldots,K,$
are as diagonal as possible. This can be formulated as minimizing the total off-diagonal energy
\[
\min_{V\in \mathbb{O}(d)}
\sum_{k=1}^K \left\|V^\top C_k V-\operatorname{diag}\left(V^\top C_k V\right)\right\|_F^2,
\]
where \(\operatorname{diag}(A)\) denotes the matrix obtained from the diagonal elements of \(A\), and $\mathbb{O}(d)$ denotes the set of all the $d\times d$ unitary matrices. The Jacobi-sweep procedure approximately minimizes this objective by applying a sequence of pairwise rotations. For each coordinate pair \((p,q)\in[d]\times[d]\), it selects a rotation matrix that rotates coordinates in dimension $p$ and $q$ with optimized angle $\theta$, i.e., \(G_{pq}(\theta)\), to decrease the joint off-diagonal energy of the transformed matrices, and updates \[ C_k \leftarrow G_{pq}(\theta)^\top C_k G_{pq}(\theta), \qquad V \leftarrow V G_{pq}(\theta). \] A full sweep visits all pairs \(1\le p<q\le d\), and repeated sweeps progressively make the matrices in \(\mathcal C\) more nearly diagonal.

Let $U$ and $V$ denote the JADE-optimal orthogonal matrices for
$\{A_k\}_{k=1}^4$ and $\{B_k\}_{k=1}^4$, respectively, so that
$A_k \approx U\,\mathrm{Diag}(a_k^{(1)},\dots,a_k^{(d_1)})\,U^\top$
and
$B_k \approx V\,\mathrm{Diag}(b_k^{(1)},\dots,b_k^{(d_2)})\,V^\top$.
We obtain
$\eta_{\mathrm{sd}}(\{A_k\}_{k=1}^4) = 0.892$ and
$\eta_{\mathrm{sd}}(\{B_k\}_{k=1}^4) = 0.845$,
i.e.\ only about 11--15\% of the squared Frobenius energy remains
off-diagonal as shown in Figure~\ref{fig:joint_diag}.

\subsection{Curvature Heterogeneity (Assumption~\ref{ass:main-ordering})}\label{app:hetero}

With the JADE-aligned $U, V$ from Appendix~\ref{app:simul_diag}, we form $w_i = \sum_{k=1}^{4} a_{k}^{(i)} b_{k}^{(i)}$ and rank by
$|w_i|$. We consider on the top-$128$ values of $|w_i|$, as the remaining entries have negligible magnitude (the $128$-th largest $|w_i|$ is less than $10^{-6}$ of $|w_1|$). Within the top-$128$ values of $|w_i|$, $q=88$ are positive as shown in Figure~\ref{fig:hetero}; their spectrum is strongly heterogeneous, $w_1 / w_{88} \approx 2.59 \times 10^{6}$, with $w_1$ alone accounting for a large fraction of the positive mode trace. 

\subsection{Gradient alignment (Assumption~\ref{ass:main-alignment})}
\label{app:gradient-alignment}

Let $u_i\in\bbR^{d_1}$ and $v_i\in\bbR^{d_2}$ denote the $i$-th columns of $U$ and $V$ in Assumption~\ref{ass:simul_diag}, respectively. 
Let $\pi(i)$ denote the index of the $i$-th largest paired curvature among $\{\sum_{k=1}^r a_k^{(j)}b_k^{(j)}:j\in[d']\}$, where $d^\prime=\min(d_1,d_2)$. 
We define $M_i=u_{\pi(i)}v_{\pi(i)}^\top$ for $i\in[q]$. 
Since $U$ and $V$ are orthogonal matrices, the matrices $\{M_i\}_{i=1}^{q}$ are orthonormal under the Frobenius inner product. To verify this assumption, we define $\calM_{i}$ as the subspace spanned by $\{M_{j}\}_{j=1}^{i}$ with $i\in[q]$ and let $\Pi_{\calM_{i}}$ denote the projection onto this subspace and the cumulative energy ratio as $\zeta(i)=\|\Pi_{\mathcal M_{i}}G\|_F^2/\|G\|_F^2$. We compute $\zeta(i)$ and the paired-diagonal decomposition below
on the same checkpoint and dense Hessian used in Appendices~\ref{app:kron}–\ref{app:hetero}.

We first re-run the same forward/backward pass that generated the dense Hessian, with identical
batch and seed, and flatten the gradient on the parameter matrix into $g \in \mathbb{R}^{d_1 d_2}$.
To compute $\zeta(i)$, we project $g$ onto the subspace spanned by
$\{\mathrm{vec}(M_1),\dots,\mathrm{vec}(M_i)\}$, where $M_i$ are the
JADE-based rank-one modes defined above.
We also compute the top $q=88$ positive
eigenpairs $(\lambda_i, \psi_i)$ of $\mathrm{mat}(\mathcal{H})$ and
verify that the resulting cumulative energy ratios are nearly identical,
confirming that the JADE-based modes are close to true Hessian eigenvectors. Figure~\ref{fig:main-ordering-alignment} plots $\zeta(i)$ as a function of
$i$; at $i = q = 88$ we obtain $\zeta(q) = 0.871$, so the active subspace
captures $87.1\%$ of the gradient Frobenius energy.

The main-text diagnostic for Assumption~\ref{ass:main-alignment} supports that most of the gradient energy lies in the positive row/column subspace, with $\zeta(q)=0.871$.
However, this does not imply the stronger paired-diagonal form $G = U_{q}\Sigma V_{q}^\top$, where $\Sigma$ is diagonal; the latter would require the JADE-basis coefficients
$\Gamma=U^\top G V$ under Assumption~\ref{ass:main-ordering} to be concentrated on entries $\Gamma_{ii}$.
We therefore check whether the Hessian quadratic form relevant to the sharpness is nevertheless determined by the paired-diagonal positive component. Under Assumption~\ref{ass:main-ordering}, let $\Gamma= U^\top G V$, $\Gamma_{ij} =\gamma_{ij}$,  and $\Lambda_{ij}=\sum_{k=1}^r a_{k}^{(i)} b_{k}^{(j)}$. In this paired basis, the gradient norm and Hessian quadratic form become
\begin{equation*}
\langle G,\mathcal{H}[G]\rangle=\sum_{i,j}\Lambda_{ij}\gamma_{ij}^2,
\qquad\|G\|_F^2=\sum_{i,j}\gamma_{ij}^2.
\end{equation*}
Let $\mathcal{A} = \{i : w_i > 0\}, w_i = \Lambda_{ii}$ denote the positive set from Assumption~\ref{ass:main-ordering}. We partition both sums into three segments: (A)~the intended
paired-diagonal component ($i = j$, $i \in \mathcal{A}$),     (B)~the off-diagonal component ($i \neq j$), (C)~the non-positive diagonal component ($i = j$, $i \notin \mathcal{A}$).
Since segment~(C) contributes approximately $0\%$, we only report the first two groups in Table~\ref{tab:relaxed-a33-decomposition}.

\begin{table}[h]
\centering
\caption{Hessian quadratic-form decomposition}
\label{tab:relaxed-a33-decomposition}
\begin{tabular}{lc}
\toprule
segment & share of $\langle G, \mathcal{H}[G]\rangle$ \\
\midrule
(A) intended paired-diagonal & 88.0\%  \\
(B) off paired-diagonal & 12.0\% \\
\bottomrule
\end{tabular}
\end{table}

Table~\ref{tab:relaxed-a33-decomposition} shows that the Hessian is
dominated by the paired-diagonal positive component: it accounts for $88.0\%$ of
$\langle G,\mathcal H[G]\rangle$, while the off paired-diagonal component
accounts for only $12.0\%$. Thus, although $G$ is not exactly paired-diagonal in the JADE
basis, the curvature cost of the gradient-aligned direction is still primarily
controlled by the same paired-diagonal mechanism.
This shows that our structured quadratic model used in the theory is a
reasonable idealization.

\subsection{Synthetic GD vs.\ Adam vs.\ Muon under the structured quadratic model}
\label{app:synthetic}

The synthetic experiments in Figure~\ref{fig:synth-muon-gd-signgd} quantify the NDS
ratio and the optimal-step quadratic-model decrease ratio for Muon, Adam,
and GD on random instances of the structured quadratic model satisfying
Assumptions~\ref{ass:block-curvature}--\ref{ass:main-alignment}.
Each instance is parameterised by $(d_1, d_2, q, \alpha_w, \alpha_\sigma)$
with defaults $d_1 = d_2 = 256$, $q = 88$, $\alpha_w = 1.3$,
$\alpha_\sigma = 0.5$.
We sample two random orthogonal matrices
$U \in \mathbb{R}^{d_1 \times d_1}$ and
$V \in \mathbb{R}^{d_2 \times d_2}$,
and construct the curvature--gradient pair as
$w_i \propto i^{-\alpha_w}$ and $\sigma_i = w_i^{\alpha_\sigma}$
for $i \in [q]$, with $w_i = 0$ for $i > q$.
Thus only $q = 88$ modes carry positive curvature, matching the
effective dimensionality observed empirically in
Appendix~\ref{app:hetero}.
The power-law spectrum generalises the heterogeneous property in
Assumption~\ref{ass:main-ordering}.

All three optimizers are initialised at $Y_0 = 0$ and use exact
line-search step sizes as described in Section~\ref{sec:case-study-kronecker}. Figure~\ref{fig:ratio-a} reports the NDS ratio normalised by
GD's value. Muon achieves an NDS roughly $11\times$ smaller than
GD's ($0.09\times$), while Adam's NDS is about half of GD's
($0.50\times$).
Figure~\ref{fig:ratio-b} reports the loss-decrease ratio.
Muon attains a $5.60\times$ larger loss decrease than GD, whereas
Adam's advantage over GD is modest ($1.22\times$).
Notably, Adam and GD exhibit similar NDS and loss-decrease profiles,
which justifies our theoretical focus on the GD--Muon comparison
in Section~\ref{sec:case-study-kronecker}.

\section{Proof of Theorem~\ref{thm:sharpness}}
\label{app:two-group-proof}

Assumptions~\ref{ass:block-curvature}--\ref{ass:main-alignment} reduce the local quadratic model
\begin{align*}
\mathcal Q(Y)
    =\calL(W_0)-\langle G,Y\rangle+1/2\cdot\langle Y,\mathcal H[Y]\rangle 
\end{align*}
to a low-dimensional model on the positive-curvature subspace. Specifically, the gradient and Hessian are simultaneously represented along paired rank-one modes $\{M_i\}_{i=1}^q$ as
\begin{equation*}
     G=\sum_{i=1}^q \sigma_iM_i, \text{ and }\mathcal H[M_i]=w_iM_i.
\end{equation*}
Here, the modes $M_i$ have positive and heterogeneous curvatures $w_i$, and higher-curvature modes tend to carry larger gradient energy, as formalized by Assumptions~\ref{ass:block-curvature}, \ref{ass:simul_diag}, and \ref{ass:main-alignment}. We refer to this span as the positive-mode subspace. Under the two-group profile in Assumption~\ref{ass:main-ordering}, the positive curvatures take the form
\begin{equation*}
    w_i=w_{\rm H},\quad i=1,\ldots,m,
    \qquad
    w_i=w_{\rm L},\quad i=m+1,\ldots,q,
\end{equation*}
where $w_{\rm H}>w_{\rm L}>0$. Similarly, the coefficients in
$G=\sum_{i=1}^q \sigma_iM_i$ follow the same grouping:
\begin{equation*}
    \sigma_i=\sigma_{\rm H},\quad i=1,\ldots,m,
    \qquad
    \sigma_i=\sigma_{\rm L},\quad i=m+1,\ldots,q,
\end{equation*}
where $\sigma_{\rm H}>\sigma_{\rm L}>0$, as specified in Assumption~\ref{ass:main-alignment}. For convenience, define
\begin{equation*}
    \alpha=\frac{m}{q},
    \qquad
    \beta=1-\alpha=1-\frac{m}{q},
    \qquad
    \rho=\frac{w_{\rm H}}{w_{\rm L}}>1,
    \qquad
    \tau=\frac{\sigma_{\rm H}}{\sigma_{\rm L}}>1.
\end{equation*}

In the following, we focus on the non-degenerate case in which the residuals defining the normalized update directions remain nonzero throughout the considered time horizon. If a method reaches the quadratic minimizer earlier, then its suboptimality is zero, and the comparison follows from the same formulas under the natural stopping convention. Our proof proceeds in three main steps:
\begin{itemize}[leftmargin=2em]
    \item \textbf{Step 1: Residual Representation and exact line-search stepsize.} 
    We derive the residual dynamics of Muon and GD in the positive mode subspace 
    and establish the exact line-search stepsize.
    
    \item \textbf{Step 2: Sharpness Comparison.} 
    Using the residual representation, we compute the \ac{nds} for Muon and GD, and compare their average sharpness across the trajectory.
    
    \item \textbf{Step 3: Loss Decrease Comparison.} 
    We analyze the cumulative loss decrease over the horizon by computing the 
    exact line-search updates and the suboptimality for both Muon and GD. 
    This step provides a closed-form comparison of their total loss decrease.
\end{itemize}

\noindent \textbf{Step 1: Residual representation and exact line-search stepsize.}
We begin by showing that all iterates $Y_t^{\opt}$ and residuals $G-\mathcal H[Y_t^{\opt}]$ remain in the positive-mode subspace for $\opt\in\{\muon,\gd\}$, so that the dynamics can be tracked through scalar residual coefficients. Specifically, let $\mathcal{M}_q=\operatorname{span}(\{M_i\}_{i=1}^q)$. For each
$\opt\in\{\muon,\gd\}$ and every $t=0,\ldots,T$, we will prove
\begin{itemize}[leftmargin=1em]
    \item The parameter $Y_t$ and the residual $G-\mathcal H[Y_t^{\opt}]$ are in $\mathcal{M}_q$, i.e., $Y_t^{\opt}\in\mathcal{M}_q,
    G-\mathcal H[Y_t^{\opt}]\in\mathcal{M}_q.$
    \item For every $t=0,\ldots,T-1$, the update increment satisfies $\eta_t^{\opt}Z_t^{\opt}\in\mathcal{M}_q.$
\end{itemize}
We prove these results by induction on $t$. At $t=0$, we have $Y_0^{\opt}=0\in\mathcal{M}_q$.
Since $G=\sum_{i=1}^q\sigma_iM_i\in\mathcal{M}_q$, we also have $G-\mathcal H[Y_0^{\opt}]=G\in\mathcal{M}_q$. Assume $Y_t^{\opt}\in\mathcal{M}_q$ for $t\geq 0$. Then there exist coefficients $y_{i,t}^{\opt}$ such that
\begin{equation*}
    Y_t^{\opt}=\sum_{i=1}^q y_{i,t}^{\opt}M_i.
\end{equation*}
Using $\mathcal H[M_i]=w_iM_i$, we get $\mathcal H[Y_t^{\opt}]=\sum_{i=1}^q w_i y_{i,t}^{\opt}M_i$. Therefore, we obtain
\begin{equation*}
    G-\mathcal H[Y_t^{\opt}]
    =
    \sum_{i=1}^q
    \left(\sigma_i-w_i y_{i,t}^{\opt}\right)M_i
    =
    \sum_{i=1}^q r_{i,t}^{\opt}M_i
    \in\mathcal{M}_q.
\end{equation*}

For Muon, the update direction  is the spectral sign of the residual, which takes the form 
\begin{equation*}
    Z_t^{\muon}
    =
    \sum_{i=1}^q \sgn(r_{i,t}^{\muon})M_i
    \in\mathcal{M}_q,
\end{equation*}
where $\sgn(\cdot)$ takes the sign of the input. For GD, the update direction is the residual itself,
\begin{equation*}
    Z_t^{\gd}
    =
    \sum_{i=1}^q r_{i,t}^{\gd}M_i
    \in\mathcal{M}_q.
\end{equation*}
Therefore, in both cases, we have that
\begin{equation*}
\eta_t^{\opt}Z_t^{\opt}\in\mathcal{M}_q,\qquad Y_{t+1}^{\opt}=Y_t^{\opt}+\eta_t^\opt Z_t^{\opt}\in\mathcal{M}_q.
\end{equation*}
This completes the induction.  Consequently, writing
$Y_t^{\opt}=\sum_{i=1}^q y_{i,t}^{\opt}M_i$ and defining the per-mode residual $r_{i,t}^{\opt}=\sigma_i-w_i y_{i,t}^{\opt}$, we obtain
\begin{equation*}
    G-\mathcal H[Y_t^{\opt}]
    =\sum_{i=1}^q r_{i,t}^{\opt}M_i.
\end{equation*}
Here, $r_{i,t}^{\opt}$ is the stationarity residual along mode $M_i$.

With the residual representation above, we next state a lemma for the exact line-search step size $\eta_t^{\opt}$, which is heavily used in the later proof.

\begin{lemma}[Exact line-search stepsize]\label{lem:line_stepsize}
Let $\mathcal R_t=G-\mathcal H[Y_t]$ denote the residual at step $t$. For any direction $Z$, the local quadratic model satisfies
\begin{equation*}
    \mathcal Q(Y_t+\eta Z)
    =\mathcal Q(Y_t)
    -\eta\langle \mathcal R_t,Z\rangle
    +\frac{\eta^2}{2}\langle Z,\mathcal H[Z]\rangle.
\end{equation*}
Then the optimal step size is
\begin{equation*}
    \eta_t(Z)
    =\max\bigg\{\frac{\langle \mathcal R_t,Z\rangle}{\langle Z,\mathcal H[Z]\rangle},0\bigg\}. 
\end{equation*}
\end{lemma}
The proof is in Appendix~\ref{app:line_stepsize}

\vspace{5pt}
\noindent \textbf{Step 2: Sharpness comparison.}
We now use the residual representation to compare the \ac{nds} of Muon and GD. The key observation is that Muon's update direction weights all modes $M_i$ equally regardless of their residual magnitudes, whereas GD's direction concentrates on whichever group currently carries more residual energy. We compute the \ac{nds} for each optimizer in turn, then compare their averages over the trajectory.


\vspace{5pt}
\noindent {\bf\emph{\ac{nds} of Muon.}} By the definition of $Z_t^{\muon}$, its update direction is written as $Z_t^{\muon}=\sum_{i=1}^q \sgn(r_{i,t}^{\muon})M_i$.
Since the modes $M_i$ are Frobenius-orthonormal, i.e., $\langle M_i,M_j\rangle=\delta_{ij}$, and the step is non-degenerate, it gives
\begin{equation*}
    \|Z_t^{\muon}\|_F^2
    =\sum_{i=1}^q \sgn(r_{i,t}^{\muon})^2=q.
\end{equation*}
Also, we can write the $\mathcal H[Z_t^{\muon}]$ as follows
\begin{equation*}
    \mathcal H[Z_t^{\muon}]
    =
    \sum_{i=1}^q
    w_i\sgn(r_{i,t}^{\muon})M_i.
\end{equation*}
Then, the curvature term evaluates to
\begin{equation*}
\langle Z_t^{\muon},\mathcal H[Z_t^{\muon}]\rangle=\sum_{i=1}^q w_i=mw_{\rm H}+(q-m)w_{\rm L}.
\end{equation*}
Therefore, the \ac{nds} of Muon is
\begin{align}
    \Sc_F(Z_t^{\muon})
    =\frac{mw_{\rm H}+(q-m)w_{\rm L}}{q}=
    \alpha w_{\rm H}+(1-\alpha)w_{\rm L}.\label{eq:nds_m}
\end{align}
The Muon \ac{nds} is constant across steps because the spectral normalization erases all magnitude information, making the curvature seen by Muon depend only on the fixed group structure. In contrast, GD's \ac{nds} fluctuates from step to step. To track this fluctuation, we introduce the high-curvature energy share, which measures how much of the total residual energy is concentrated in the high-curvature modes.

\vspace{5pt}
\noindent {\bf\emph{\ac{nds} of GD.}} The value of \ac{nds} is summarized in the following proposition.
\begin{proposition}[\ac{nds} of GD]\label{prop:gd_sharp}
For GD with $G-\mathcal H[Y_t^{\gd}]=\sum_{i=1}^q r_{i,t}^{\gd}M_i$, we define the high-curvature energy share as
\begin{equation*}
    P_t^{\gd}=\frac{\sum_{i=1}^m (r_{i,t}^{\gd})^2}
    {\sum_{i=1}^q (r_{i,t}^{\gd})^2},
\end{equation*}
which measures the fraction of total residual energy carried 
by the high-curvature modes. Then the \ac{nds} of GD is
\begin{equation*}
    \Sc_F(Z_t^{\gd})
    =P_t^{\gd}\,w_{\rm H}+(1-P_t^{\gd})\,w_{\rm L}.
\end{equation*}
In addition, the value of $P_t^{\gd}$ alternates between two values according to 
$P_{t+1}^{\gd}=1-P_t^{\gd}$, with initial value 
\begin{equation*}
P_0^{\gd}=p=\frac{m\sigma_{\rm H}^2}
{m\sigma_{\rm H}^2+(q-m)\sigma_{\rm L}^2}.
\end{equation*}
In particular, $\Sc_F(Z_t^{\gd})$ oscillates between 
$p\,w_{\rm H}+(1-p)\,w_{\rm L}$ and 
$(1-p)\,w_{\rm H}+p\,w_{\rm L}$.
\end{proposition}
The proof is in Appendix~\ref{app:gd_sharp}. In the following, we compare the \ac{nds} of $\muon$ and $\gd$. 
Eqn.~\eqref{eq:nds_m} and Proposition~\ref{prop:gd_sharp}, the difference in average \ac{nds} is
\begin{equation*}
    \bar \Sc_T^{\gd}-\bar \Sc_T^{\muon}=(\bar p_T-\alpha)(w_{\rm H}-w_{\rm L}),
\end{equation*}
where $\bar p_T=T^{-1}\sum_{t=0}^{T-1}P_t^{\gd}$ is the time-averaged energy share. Since $P_t^{\gd}$ alternates between $p$ and $1-p$, we have
\begin{equation*}
    \bar p_T=
    \begin{cases}
        \frac12, & T \text{ even},\\[4pt]
        \frac12+\frac{p-\frac12}{T}, & T \text{ odd}.
    \end{cases}
\end{equation*}
Since $w_{\rm H}>w_{\rm L}$, it suffices to show $\bar p_T>\alpha$. The condition $\sigma_{\rm H}>\sigma_{\rm L}$ implies $p>\alpha$. If $T$ is even, $\bar p_T=1/2>\alpha$. If $T=2N+1$ is odd, then $\bar p_T=(N+p)/(2N+1)$ and
\begin{equation*}
    \bar p_T-\alpha=\frac{N(1-2\alpha)+(p-\alpha)}{2N+1}>0,
\end{equation*}
since $\alpha<1/2$ and $p>\alpha$. This proves $\bar \Sc_T^{\muon}<\bar \Sc_T^{\gd}$. 

\vspace{5pt}
\noindent \textbf{Step 3: Loss decrease comparison.}

We now turn to the cumulative loss decrease. Our strategy is to express the terminal objective gap in a common residual form, derive its closed form separately for Muon and GD, and then show that Muon has a smaller gap at every horizon.
In the following, we will analyze
\begin{align*}
    \Phi_t^{\muon}=
    \mathcal Q(Y_t^{\muon})-\mathcal Q(Y^\star), \text{ and }\Phi_t^{\gd}=
    \mathcal Q(Y_t^{\gd})-\mathcal Q(Y^\star),
\end{align*}
where $Y^*=\argmin_{Y\in\mathcal{M}_q}  \mathcal Q(Y)$ is the minimizer of the loss within the subspace $\mathcal{M}_q$. We now derive the residual dynamics for each optimizer, starting with Muon. For Muon, we have the following proposition.
\begin{proposition}
[Muon suboptimality]\label{prop:muon_gap}
Let $d_0=|\sigma_{\rm H}/w_{\rm H}-\sigma_{\rm L}/w_{\rm L}|$ denote the initial gap between the scaled residuals of the two groups, and let
\begin{equation*}
    \Gamma
    =
    \frac{|a-b|}{a+b}
    =
    \frac{
        |mw_{\rm H}-(q-m)w_{\rm L}|
    }{
        mw_{\rm H}+(q-m)w_{\rm L}
    }
\end{equation*}
be the normalized curvature imbalance between the two groups, where $a=mw_{\rm H}, b=(q-m)w_{\rm L}$. For every $T\ge1$, we have the suboptimality of Muon $\Phi_T^{\muon}$ as follows
\begin{equation*}
    \Phi_T^{\muon}=
    \mathcal Q(Y_T^{\muon})-\mathcal Q(Y^\star)
    =
    \frac{ab}{2(a+b)}
    \Gamma^{2(T-1)}
    d_0^2.
\end{equation*}
\end{proposition}
The proof is in Appendix~\ref{app:muon_gap}. This result shows that Muon has linear convergence to $Y^\star$. One-step optimization will decrease the suboptimality with coefficient $\Gamma^{2T}$ multiplicatively. We now turn to GD. Unlike Muon, GD's direction is proportional to the residual itself.
\begin{proposition}
[GD suboptimality]\label{prop:gd_gap}
Let $p=m\sigma_{\rm H}^2/(m\sigma_{\rm H}^2+(q-m)\sigma_{\rm L}^2)$ be the initial high-curvature energy share. Define the per-step contraction function as
\begin{equation*}
    C(x)=
    \frac{
        (w_{\rm H}-w_{\rm L})^2x(1-x)
    }{
        \left(w_{\rm L}+(w_{\rm H}-w_{\rm L})x\right)^2
    },
\end{equation*}
which measures how much total residual energy is retained after one GD step when the high-curvature share is $x$, and let $R=C(p)C(1-p)$ be the two-step contraction factor. Then for every $T\ge0$,
the suboptimality of GD can be written as follows.
\begin{equation*}
    \Phi_T^{\gd}
    =
    \Phi_0^{\gd} R^{T/2},
\end{equation*}
where $\Phi_0^{\gd}$ is the suboptimality at initialization, i.e., 
\begin{equation*}
    \Phi_0^{\gd}
    =
    \frac12
    \left(
        \frac{m\sigma_{\rm H}^2}{w_{\rm H}}
        +
        \frac{(q-m)\sigma_{\rm L}^2}{w_{\rm L}}
    \right).
\end{equation*}
Equivalently, the GD suboptimality contracts by the same factor $\sqrt R$ at every step:
\begin{equation*}
    \Phi_{t+1}^{\gd}=\sqrt R\,\Phi_t^{\gd}.
\end{equation*}
\end{proposition}
The proof is in Appendix~\ref{app:gd_gap}. This result shows that GD also converges linearly. One-step optimization decreases the suboptimality with $\sqrt R$ multiplicatively. With the suboptimalities for both optimizers in closed form, it remains to show that Muon's gap is strictly smaller at every horizon. We achieve this by two steps: 
\begin{itemize}[leftmargin=2em]
    \item After the first-step optimization, Muon has a lower loss than GD.
    \item During the whole optimization process, the multiplicative coefficient of Muon $\Gamma^2$ is smaller than that of GD $\sqrt{R}$.
\end{itemize}
The first step is achieved by the following proposition.
\begin{proposition}
[One-step terminal-gap comparison]\label{prop:onestep_gap}
Under Assumptions~\ref{ass:block-curvature}--\ref{ass:main-alignment} and $\rho+1>1/\alpha>1+\sigma_{\rm H}/\sigma_{\rm L}$ with $\rho=w_{\rm H}/w_{\rm L}$, we have
\begin{equation*}
    \Phi_1^{\muon}<\Phi_1^{\gd}.
\end{equation*}
\end{proposition}
The proof is in Appendix~\ref{app:onestep_gap}, where we directly calculate the parameters after first optimization step and compare the suboptimality. The second step is achieved by the following proposition.

\begin{proposition}
[Muon contracts faster after the first step]\label{prop:muon_contract}
Under Assumptions~\ref{ass:block-curvature}--\ref{ass:main-alignment} and $\rho+1>1/\alpha>1+\sigma_{\rm H}/\sigma_{\rm L}$ with $\rho=w_{\rm H}/w_{\rm L}$, we have
\begin{equation*}
    \sqrt R>\Gamma^2,
\end{equation*}
where $\Gamma=|mw_{\rm H}-(q-m)w_{\rm L}|/(mw_{\rm H}+(q-m)w_{\rm L}).$
\end{proposition}
The proof is in Appendix~\ref{app:muon_contract}, which follows from the direct calculations.
With these results, we prove that Muon has a lower suboptimality than GD as follows. From Propositions~\ref{prop:muon_gap} and~\ref{prop:gd_gap}, the suboptimalities factor as
\begin{equation*}
    \Phi_T^{\muon}=\Phi_1^{\muon}\,\Gamma^{2(T-1)},
    \qquad
    \Phi_T^{\gd}=\Phi_1^{\gd}\,R^{(T-1)/2}.
\end{equation*}
Proposition~\ref{prop:onestep_gap} gives $\Phi_1^{\muon}<\Phi_1^{\gd}$ and Proposition~\ref{prop:muon_contract} gives $\Gamma^2<\sqrt{R}$. Therefore, for every $T\ge1$,
\begin{equation*}
    \Phi_T^{\muon}=\Phi_1^{\muon}\,\Gamma^{2(T-1)}<\Phi_1^{\gd}\,R^{(T-1)/2}=\Phi_T^{\gd}.
\end{equation*}
Since $\Phi_T^{\opt}=\mathcal Q(Y_T^{\opt})-\mathcal Q(Y^\star)$, this implies 
\begin{equation*}
\mathcal Q(Y_T^{\muon})<\mathcal Q(Y_T^{\gd}).
\end{equation*}

Therefore, we conclude the proof of Theorem~\ref{thm:sharpness}.

\section{Proofs of Supporting Propositions and Lemmas}
\subsection{Proof of Lemma~\ref{lem:line_stepsize}}\label{app:line_stepsize}
By the definition of the quadratic loss function $\mathcal Q$, we have
\begin{equation*}
\begin{aligned}
    \mathcal Q(Y_t+\eta Z)
    &=
    \mathcal L(W_{0})
    -
    \langle G,Y_t+\eta Z\rangle
    +
    \frac12
    \langle Y_t+\eta Z,\mathcal H[Y_t+\eta Z]\rangle.
\end{aligned}
\end{equation*}
Expanding the linear term gives
\begin{equation*}
    -\langle G,Y_t+\eta Z\rangle
    =
    -\langle G,Y_t\rangle
    -
    \eta\langle G,Z\rangle.
\end{equation*}
Since $\mathcal H$ is linear and self-adjoint under the Frobenius inner product, the quadratic term expands as
\begin{equation*}
\begin{aligned}
    \frac12\cdot
    \langle Y_t+\eta Z,\mathcal H[Y_t+\eta Z]\rangle
    &=
    \frac12\cdot\langle Y_t,\mathcal H[Y_t]\rangle
    +
    \eta\cdot\langle Z,\mathcal H[Y_t]\rangle
    +
    \frac{\eta^2}{2}\cdot\langle Z,\mathcal H[Z]\rangle.
\end{aligned}
\end{equation*}
Combining and grouping the terms linear in $\eta$ yields
\begin{equation*}
\begin{aligned}
    \mathcal Q(Y_t+\eta Z)
    &=
    \mathcal Q(Y_t)
    -
    \eta\langle G,Z\rangle
    +
    \eta\langle \mathcal H[Y_t],Z\rangle
    +
    \frac{\eta^2}{2}\langle Z,\mathcal H[Z]\rangle\\
    &=
    \mathcal Q(Y_t)
    -
    \eta\langle G-\mathcal H[Y_t],Z\rangle
    +
    \frac{\eta^2}{2}\langle Z,\mathcal H[Z]\rangle\\
    &=
    \mathcal Q(Y_t)
    -
    \eta\langle \mathcal R_t,Z\rangle
    +
    \frac{\eta^2}{2}\langle Z,\mathcal H[Z]\rangle.
\end{aligned}
\end{equation*}
Hence the decrease along $Z$ is
\begin{equation*}
    \mathcal Q(Y_t)-\mathcal Q(Y_t+\eta Z)
    =
    \eta\langle \mathcal R_t,Z\rangle
    -
    \frac{\eta^2}{2}\langle Z,\mathcal H[Z]\rangle.
\end{equation*}
Maximizing the per-step decrease $\eta\langle \mathcal R_t,Z\rangle-\frac{\eta^2}{2}\langle Z,\mathcal H[Z]\rangle$ over $\eta$ gives 
\begin{equation*}
\eta_t(Z)=\max\bigg\{\frac{\langle \mathcal R_t,Z\rangle}{\langle Z,\mathcal H[Z]\rangle},0\bigg\}.
\end{equation*}
Thus, we conclude the proof of Lemma~\ref{lem:line_stepsize}.

\subsection{Proof of Proposition~\ref{prop:gd_sharp}}\label{app:gd_sharp}
Let $A_t=\sum_{i=1}^m (r_{i,t}^{\gd})^2$ and 
$B_t=\sum_{i=m+1}^q (r_{i,t}^{\gd})^2$ denote the 
residual energy in the high- and low-curvature groups, respectively. With $S_t=A_t+B_t$, the high-curvature energy share $P_t^{\gd}$ is $P_t^{\gd}=A_t/S_t$.

We first derive the \ac{nds} expression. Since the GD direction is the same as that of the residual itself, i.e.,  
$Z_t^{\gd}=G-\mathcal H[Y_t^{\gd}]=\sum_{i=1}^q r_{i,t}^{\gd}M_i$. 
By $\mathcal H[M_i]=w_iM_i$ and the 
Frobenius orthonormality $\langle M_i,M_j\rangle=\delta_{ij}$, 
we have that
\begin{equation*}
    \|Z_t^{\gd}\|_F^2
    =\sum_{i=1}^q (r_{i,t}^{\gd})^2
    =A_t+B_t=S_t,\quad 
    \langle Z_t^{\gd},\mathcal H[Z_t^{\gd}]\rangle
    =\sum_{i=1}^q w_i(r_{i,t}^{\gd})^2
    =w_{\rm H}A_t+w_{\rm L}B_t.
\end{equation*}
Their ratio gives
\begin{equation*}
    \Sc_F(Z_t^{\gd})
    =\frac{w_{\rm H}A_t+w_{\rm L}B_t}{S_t}
    =w_{\rm H}\,P_t^{\gd}+w_{\rm L}(1-P_t^{\gd}).
\end{equation*}

We then derive the Energy-share recursion. By Lemma~\ref{lem:line_stepsize}, the exact line-search step size is
\begin{equation*}
    \eta_t^{\gd}
    =\max\bigg\{\frac{\langle \mathcal R_t^{\gd},Z_t^{\gd}\rangle}
    {\langle Z_t^{\gd},\mathcal H[Z_t^{\gd}]\rangle},0\bigg\}.
\end{equation*}
Since the GD direction equals the residual, 
the numerator is 
$\langle \mathcal R_t^{\gd},Z_t^{\gd}\rangle
=\|\mathcal R_t^{\gd}\|_F^2=S_t$ 
and the denominator is 
$\langle Z_t^{\gd},\mathcal H[Z_t^{\gd}]\rangle
=\Sc_F(Z_t^{\gd})\cdot S_t$, 
giving $\eta_t^{\gd}=1/\Sc_F(Z_t^{\gd})$.
With $y_{i,t} = \langle M_i, Y_t^{\gd}\rangle$ ,the coordinate update is 
$y_{i,t+1}^{\gd}=y_{i,t}^{\gd}+\eta_t^{\gd}\,r_{i,t}^{\gd}$, 
so the residual at the next step is
\begin{equation*}
    r_{i,t+1}^{\gd}
    =\sigma_i-w_i y_{i,t+1}^{\gd}
    =r_{i,t}^{\gd}-w_i\,\eta_t^{\gd}\,r_{i,t}^{\gd}
    =r_{i,t}^{\gd}\left(1-\frac{w_i}{\Sc_F(Z_t^{\gd})}\right).
\end{equation*}
Evaluating the contraction factor $w_i/\Sc_F(Z_t^{\gd})$ for each group gives
\begin{equation*}
    1-\frac{w_{\rm H}}{\Sc_F(Z_t^{\gd})}
    =-\frac{(w_{\rm H}-w_{\rm L})(1-P_t^{\gd})}{\Sc_F(Z_t^{\gd})},
    \qquad
    1-\frac{w_{\rm L}}{\Sc_F(Z_t^{\gd})}
    =\frac{(w_{\rm H}-w_{\rm L})P_t^{\gd}}{\Sc_F(Z_t^{\gd})}.
\end{equation*}
The opposite signs reveal the zig-zag mechanism that GD overshoots 
the high-curvature modes (negative factor) while undershooting 
the low-curvature ones (positive factor), causing the energy to 
alternate between groups. Squaring and summing within each group yields
\begin{equation*}
    A_{t+1}=A_t
    \frac{(w_{\rm H}-w_{\rm L})^2(1-P_t^{\gd})^2}{\Sc_F(Z_t^{\gd})^2},
    \qquad
    B_{t+1}=B_t
    \frac{(w_{\rm H}-w_{\rm L})^2(P_t^{\gd})^2}{\Sc_F(Z_t^{\gd})^2}.
\end{equation*}
Using $A_t=P_t^{\gd}S_t$ and $B_t=(1-P_t^{\gd})S_t$, we obtain
\begin{equation*}
    P_{t+1}^{\gd}
    =\frac{P_t^{\gd}(1-P_t^{\gd})^2}
    {P_t^{\gd}(1-P_t^{\gd})^2+(1-P_t^{\gd})(P_t^{\gd})^2}
    =1-P_t^{\gd}.
\end{equation*}
At initialization $r_{i,0}^{\gd}=\sigma_i$, so 
$P_0^{\gd}=m\sigma_{\rm H}^2/(m\sigma_{\rm H}^2+(q-m)\sigma_{\rm L}^2)=p$. 
The alternation $P_t^{\gd}\in\{p,1-p\}$ then gives the stated 
oscillation of $\Sc_F(Z_t^{\gd})$. Thus, we conclude the proof of Proposition~\ref{prop:gd_sharp}.

\subsection{Proof of Proposition~\ref{prop:muon_gap}}\label{app:muon_gap}
We start by describing the suboptimality of any parameter $Y_t$.
\begin{lemma}[Residual form of the suboptimality]\label{lem:terminal_gap_residual}
Let $Y^\star=\sum_{i=1}^q y_{i}^\star M_i$ with $y_i^\star=\sigma_i/w_i$ for $i\in[q]$ be the minimizer of $\mathcal Q$ in the positive subspace. For any iterate $Y_t=\sum_{i=1}^q y_{i,t}M_i$ with residual coefficients $r_{i,t}=\sigma_i-w_i y_{i,t}$, the suboptimality is
\begin{equation*}
    \Phi_t=\mathcal Q(Y_t)-\mathcal Q(Y^\star)=\frac12\sum_{i=1}^q\frac{r_{i,t}^2}{w_i}.
\end{equation*}
\end{lemma}

\begin{proof}
In positive mode, the quadratic loss of $Y=\sum_{i=1}^q y_{i}M_i$ can be written as 
\begin{equation*}
\mathcal Q(Y)=\mathcal L(W_{0})-\sum_{i=1}^q\sigma_i y_i+\frac12\sum_{i=1}^q w_i y_i^2.
\end{equation*}
The first-order condition gives $y_i^\star=\sigma_i/w_i$. Completing the square mode by mode yields
\begin{equation*}
    \mathcal Q(Y)-\mathcal Q(Y^\star)
    =\frac12\sum_{i=1}^q w_i\left(y_i-\frac{\sigma_i}{w_i}\right)^2
    =\frac12\sum_{i=1}^q\frac{(\sigma_i-w_i y_i)^2}{w_i}
    =\frac12\sum_{i=1}^q\frac{r_i^2}{w_i}.
\end{equation*}
Thus, we conclude the proof of Lemma~\ref{lem:terminal_gap_residual}.
\end{proof}
This residual form allows us to track the suboptimality through the scalar residual coefficients alone.  Then we derive the dynamics of $r_{i,t}^{\muon}$ via the following proposition.
\begin{proposition}   
[Muon residual dynamics]\label{prop:muon_residual_dynam}
Define the scaled residual $c_{i,t}^{\muon}= |r_{i,t}^{\muon}|/w_i$. With Assumption~\ref{ass:main-ordering}, there exist $c_{{\rm H},t}$ and $c_{{\rm L},t}$ such that
\begin{equation*}
    c_{i,t}^{\muon}=c_{{\rm H},t},
    \quad i=1,\ldots,m,
    \qquad
    c_{i,t}^{\muon}=c_{{\rm L},t},
    \quad i=m+1,\ldots,q.
\end{equation*}
Let $a=mw_{\rm H}, b=(q-m)w_{\rm L}$. Then the exact line-search stepsize of Muon is
\begin{equation*}
    \eta_t^{\muon}
    =
    \frac{
        a c_{{\rm H},t}
        +
        b c_{{\rm L},t}
    }{
        a+b
    },
\end{equation*}
and the scaled residuals contract as
\begin{equation*}
    c_{{\rm H},t+1}
    =
    |c_{{\rm H},t}-\eta_t^{\muon}|,
    \qquad
    c_{{\rm L},t+1}
    =
    |c_{{\rm L},t}-\eta_t^{\muon}|.
\end{equation*}
\end{proposition}

\begin{proof}
For Muon, the update direction is written as $Z_t^{\muon}=\sum_{i=1}^q \sgn(r_{i,t}^{\muon})M_i$. Thus, the updates along different modes $M_i$ are symmetric within each group of Assumption~\ref{ass:main-ordering}. Thus, $c_{i,t}^{\muon}=c_{{\rm H},t}$ for $ i=1,\ldots,m$, and $c_{i,t}^{\muon}=c_{{\rm L},t}$, for $i=m+1,\ldots,q$. By the exact line-search identity, the stepsize is shown as
\begin{equation*}
    \eta_t^{\muon}
    =
    \frac{
        \langle \mathcal R_t^{\muon},Z_t^{\muon}\rangle
    }{
        \langle Z_t^{\muon},\mathcal H[Z_t^{\muon}]\rangle
    }.
\end{equation*}
We then calculate the numerator and denominator, respectively. The numerator part is
\begin{equation*}
\langle \mathcal R_t^{\muon},Z_t^{\muon}\rangle
    =
    \bigg\langle
        \sum_{i=1}^q r_{i,t}^{\muon}M_i,
        \sum_{i=1}^q \sgn(r_{i,t}^{\muon})M_i
    \bigg\rangle=
    \sum_{i=1}^q |r_{i,t}^{\muon}|.
\end{equation*}
The denominator part is written as $\langle Z_t^{\muon},\mathcal H[Z_t^{\muon}]\rangle=\sum_{i=1}^q w_i=a+b$.
Thus, the stepsize can be rewritten as $\eta_t^{\muon}=\sum_{i=1}^q |r_{i,t}^{\muon}|/(a+b)$. Since $|r_{i,t}^{\muon}|=w_i c_{i,t}^{\muon}$, we get
\begin{equation*}
    \eta_t^{\muon}
    =
    \frac{\sum_{i=1}^q w_i c_{i,t}^{\muon}}{a+b}
    =
    \frac{
        a c_{{\rm H},t}
        +
        b c_{{\rm L},t}
    }{
        a+b
    }.
\end{equation*}

Then the update of residual $r_{i,t}^{\muon}$ is
\begin{equation*}
    r_{i,t+1}^{\muon}
    =
    r_{i,t}^{\muon}
    -
    w_i\eta_t^{\muon}\sgn(r_{i,t}^{\muon}).
\end{equation*}
Taking absolute values gives
\begin{equation*}
\begin{aligned}
|r_{i,t+1}^{\muon}|=\left||r_{i,t}^{\muon}|-w_i\eta_t^{\muon}\right|=w_i\left|c_{i,t}^{\muon}-\eta_t^{\muon}\right|.
\end{aligned}
\end{equation*}
Dividing by $w_i$ yields
\begin{equation*}
    c_{i,t+1}^{\muon}
    =
    |c_{i,t}^{\muon}-\eta_t^{\muon}|.
\end{equation*}
Thus, we conclude the proof of Proposition~\ref{prop:muon_residual_dynam}.
\end{proof}

At initialization, $c_{{\rm H},0}=\sigma_{\rm H}/w_{\rm H}$ and $c_{{\rm L},0}=\sigma_{\rm L}/w_{\rm L}$. Since $\eta_0^{\muon}$ is a convex combination of $c_{{\rm H},0}$ and $c_{{\rm L},0}$, Proposition~\ref{prop:muon_residual_dynam} gives the followings.
\begin{equation*}
    c_{{\rm H},1}
    =
    |c_{{\rm H},0}-\eta_0^{\muon}|
    =
    \frac{b}{a+b}
    |c_{{\rm H},0}-c_{{\rm L},0}|
    =
    \frac{b}{a+b}d_0,
\end{equation*}
and
\begin{equation*}
    c_{{\rm L},1}
    =
    |c_{{\rm L},0}-\eta_0^{\muon}|
    =
    \frac{a}{a+b}
    |c_{{\rm H},0}-c_{{\rm L},0}|
    =
    \frac{a}{a+b}d_0.
\end{equation*}
Therefore, we have
\begin{equation*}
    |c_{{\rm H},1}-c_{{\rm L},1}|
    =
    \frac{|a-b|}{a+b}d_0
    =
    \Gamma d_0.
\end{equation*}
Repeating the argument, the gap between the two groups contracts by $\Gamma$ at every subsequent step, yielding for $t\ge1$,
\begin{equation*}
    c_{{\rm H},t}
    =
    \frac{b}{a+b}\Gamma^{t-1}d_0,
    \qquad
    c_{{\rm L},t}
    =
    \frac{a}{a+b}\Gamma^{t-1}d_0.
\end{equation*}

By Lemma~\ref{lem:terminal_gap_residual}, we have
\begin{equation*}
\begin{aligned}
    \Phi_T^{\muon}
    =
    \frac12
    \left(
        mw_{\rm H}c_{{\rm H},T}^2
        +
        (q-m)w_{\rm L}c_{{\rm L},T}^2
    \right)=
    \frac12
    \left(
        a c_{{\rm H},T}^2
        +
        b c_{{\rm L},T}^2
    \right).
\end{aligned}
\end{equation*}
Substituting the closed form of $c_{{\rm H},T}$ and $c_{{\rm L},T}$, we have the closed form of $\Phi_T^{\muon}$ as follows.
\begin{equation*}
\begin{aligned}
    \Phi_T^{\muon}
    &=
    \frac12
    \left[
        a
        \left(
            \frac{b}{a+b}\Gamma^{T-1}d_0
        \right)^2
        +
        b
        \left(
            \frac{a}{a+b}\Gamma^{T-1}d_0
        \right)^2
    \right]\\
    &=
    \frac12
    \frac{
        ab^2+ba^2
    }{
        (a+b)^2
    }
    \Gamma^{2(T-1)}d_0^2=
    \frac{ab}{2(a+b)}
    \Gamma^{2(T-1)}d_0^2.
\end{aligned}
\end{equation*}
Thus, we conclude the proof of Proposition~\ref{prop:muon_gap}.

\subsection{Proof of Proposition~\ref{prop:gd_gap}}\label{app:gd_gap}
Follow Lemma~\ref{lem:terminal_gap_residual}, we define $Y_t^{\gd}=\sum_{i=1}^q y_{i,t}^{\gd}M_i$ with residual coefficients $r_{i,t}^{\gd}=\sigma_i-w_i y_{i,t}^{\gd}$. Let $A_t=\sum_{i=1}^m (r_{i,t}^{\gd})^2$, $B_t=\sum_{i=m+1}^q (r_{i,t}^{\gd})^2$, $S_t=A_t+B_t$, and $P_t=A_t/S_t$. From Proposition~\ref{prop:gd_sharp}, $P_{t+1}=1-P_t$ and $S_{t+1}=C(P_t)S_t$. By Lemma~\ref{lem:terminal_gap_residual}, the suboptimality can be written as
\begin{equation*}
\begin{aligned}
    \Phi_t^{\gd}=
    \frac12
    \left(
        \frac{A_t}{w_{\rm H}}
        +
        \frac{B_t}{w_{\rm L}}
    \right)=
    \frac{S_t}{2}
    \left(
        \frac{P_t}{w_{\rm H}}
        +
        \frac{1-P_t}{w_{\rm L}}
    \right).
\end{aligned}
\end{equation*}
Moreover, we have
\begin{align*}
    \frac{\Phi_{t+1}^{\gd}}{\Phi_t^{\gd}}
    &=
    C(P_t)
    \cdot \frac{
        \frac{1-P_t}{w_{\rm H}}
        +
        \frac{P_t}{w_{\rm L}}
    }{
        \frac{P_t}{w_{\rm H}}
        +
        \frac{1-P_t}{w_{\rm L}}
    }.
\end{align*}
Let $x=P_t$. Substituting $C(x)$ gives
\begin{align*}
    C(x)
    \frac{
        \frac{1-x}{w_{\rm H}}
        +
        \frac{x}{w_{\rm L}}
    }{
        \frac{x}{w_{\rm H}}
        +
        \frac{1-x}{w_{\rm L}}
    }
    &=
    \frac{
        (w_{\rm H}-w_{\rm L})^2x(1-x)
    }{
        \left(w_{\rm L}+(w_{\rm H}-w_{\rm L})x\right)^2
    }
    \frac{
        w_{\rm L}(1-x)+w_{\rm H}x
    }{
        w_{\rm L}x+w_{\rm H}(1-x)
    }\\
    &=
    \frac{
        (w_{\rm H}-w_{\rm L})^2x(1-x)
    }{
        \left(w_{\rm L}+(w_{\rm H}-w_{\rm L})x\right)
        \left(w_{\rm H}-(w_{\rm H}-w_{\rm L})x\right)
    }.
\end{align*}
The last expression is symmetric under $x\mapsto 1-x$. Therefore, it is the same for $x=p$ and $x=1-p$. Its square is $C(p)C(1-p)=R$. Hence
\begin{equation*}
    \frac{\Phi_{t+1}^{\gd}}{\Phi_t^{\gd}}
    =
    \sqrt R.
\end{equation*}
By induction, we have
\begin{equation*}
    \Phi_T^{\gd}
    =
    \Phi_0^{\gd}R^{T/2}.
\end{equation*}
Thus, we conclude the proof of Proposition~\ref{prop:gd_gap}.

\subsection{Proof of Proposition~\ref{prop:onestep_gap}}\label{app:onestep_gap}
By direct calculations, we have that
\begin{equation*}
    \Phi_1^{\muon}
    =\frac{q\sigma_{\rm L}^2}{2w_{\rm L}}\cdot
    \frac{
        \alpha\beta\rho
    }{
        \alpha\rho+\beta
    }
    \left(
        1-\frac{\tau}{\rho}
    \right)^2,\quad
    \Phi_1^{\gd}
    =\frac{q\sigma_{\rm L}^2}{2w_{\rm L}}\cdot
    \frac{
        (\rho-1)^2\alpha\beta\tau^2
        \left(
            \beta/\rho+\alpha\tau^2
        \right)
    }{
        \left(
            \alpha\rho\tau^2+\beta
        \right)^2
    },
\end{equation*}
where $\tau=\sigma_{\rm H}/\sigma_{\rm L}$, $\alpha=m/q$ is defined in Assumption~\ref{ass:main-ordering}, and $\beta=1-\alpha$. With some calculations, we have that
\begin{align}
    \Phi_1^{\gd}-\Phi_1^{\muon}=\frac{q\sigma_{\rm L}^2\rho(\tau-1)}{2w_{\rm L}(\alpha\rho\tau^2+\beta)^2\rho(\alpha\rho+\beta)}\cdot\left[
        2\alpha\rho\tau^2
        -
        \alpha\tau^3
        -
        \alpha\tau^2
        +
        \beta\rho\tau
        +
        \beta\rho
        -
        2\beta\tau
    \right].\label{eq:diff_1}
\end{align}

Since $\rho>0$ and $\tau>1$, it remains to show that the second term on the right-hand side of Eqn.~\eqref{eq:diff_1} is positive. We rewrite it as
\begin{align}
    2\alpha\rho\tau^2
    -
    \alpha\tau^3
    -
    \alpha\tau^2
    +
    \beta\rho\tau
    +
    \beta\rho
    -
    2\beta\tau
    =
    \alpha\tau^2(2\rho-\tau-1)
    +
    \beta\left(\rho(\tau+1)-2\tau\right).\label{eq:diff_2}
\end{align}
Since $1/(1+\rho)<\alpha$, we have $1<\tau<\rho$, and then
$2\rho-\tau-1>\rho-1>0.$
Thus the first term in the right-hand side of Eqn.~\eqref{eq:diff_2} is positive. For the second term, if $\rho\ge2$, then
\begin{equation*}
    \rho(\tau+1)-2\tau
    =
    \rho+\tau(\rho-2)>0.
\end{equation*}
If $1<\rho<2$, then the function $\rho+\tau(\rho-2)$ is decreasing in $\tau$, and its minimum over $\tau<\rho$ is larger than its value at $\tau=\rho$:
\begin{equation*}
    \rho+\rho(\rho-2)
    =
    \rho(\rho-1)>0.
\end{equation*}
Therefore
\begin{equation*}
    \rho(\tau+1)-2\tau>0.
\end{equation*}
Thus, we have that $
    \Phi_1^{\muon}<\Phi_1^{\gd}$, which conclude the proof of Proposition~\ref{prop:onestep_gap}.

\subsection{Proof of Proposition~\ref{prop:muon_contract}}\label{app:muon_contract}
The proof proceeds in three parts: we first simplify $\sqrt{R}$ 
into a monotone function of $p(1-p)$, then show $p(1-p)>\alpha(1-\alpha)=\alpha\beta$ 
under the assumed conditions, and finally verify that the resulting 
lower bound exceeds $\Gamma^2$.

\emph{Part 1: Simplifying $\sqrt{R}$.}
Using $\rho=w_{\rm H}/w_{\rm L}$, we rewrite the contraction 
function $C(x)$ in Proposition~\ref{prop:gd_gap} as
\begin{equation*}
    C(x)
    =\frac{(w_{\rm H}-w_{\rm L})^2x(1-x)}{(w_{\rm L}+(w_{\rm H}-w_{\rm L})x)^2}
    =\frac{(\rho-1)^2x(1-x)}{(1+(\rho-1)x)^2},
\end{equation*}
where we have divided numerator and denominator by $w_{\rm L}^2$.
Therefore
\begin{equation*}
    R=C(p)\,C(1-p)
    =\frac{(\rho-1)^4[p(1-p)]^2}
    {\big(1+(\rho-1)p\big)^2\,\big(\rho-(\rho-1)p\big)^2}.
\end{equation*}
Taking square roots gives
\begin{equation*}
    \sqrt{R}
    =\frac{(\rho-1)^2\,p(1-p)}
    {\big(1+(\rho-1)p\big)\,\big(\rho-(\rho-1)p\big)}.
\end{equation*}
To simplify the denominator, we have that $(1+(\rho-1)p\big)\,(\rho-(\rho-1)p\big)=\rho+(\rho-1)^2p(1-p)$. Substituting back, we simplify the expression of $\sqrt{R}$ as follows.
\begin{equation*}
    \sqrt{R}
    =\frac{(\rho-1)^2\,p(1-p)}
    {\rho+(\rho-1)^2\,p(1-p)}
    =F\big(p(1-p)\big),
\end{equation*}
where $F(v)=(\rho-1)^2v\,/\,(\rho+(\rho-1)^2v)$. Since the 
numerator is increasing in $v$ and the denominator is also increasing 
but starts from the positive constant $\rho$, the function $F$ is 
strictly increasing for $v>0$.

\emph{Part 2: Showing $p(1-p)>\alpha\beta$.}
We express $\Gamma$ and $p$ in terms of $\alpha,\beta,\rho,\tau$. 
For $\Gamma$ in Proposition~\ref{prop:muon_gap}, using $a=mw_{\rm H}=q\alpha\rho w_{\rm L}$ and 
$b=(q-m)w_{\rm L}=q\beta w_{\rm L}$, we get
\begin{equation*}
    \Gamma=\frac{|a-b|}{a+b}
    =\frac{\alpha\rho-\beta}{\alpha\rho+\beta},
\end{equation*}
where the absolute value is dropped because the assumption 
$\alpha>1/(1+\rho)$ implies $\alpha\rho>\beta$. For $p$, 
dividing numerator and denominator by $q\sigma_{\rm L}^2$ gives
\begin{equation*}
    p=\frac{m\sigma_{\rm H}^2}{m\sigma_{\rm H}^2+(q-m)\sigma_{\rm L}^2}
    =\frac{\alpha\tau^2}{\alpha\tau^2+\beta},
\end{equation*}
where $\tau=\sigma_{\rm H}/\sigma_{\rm L}$. We now show $p\in(\alpha,\beta)$. For the lower bound,
\begin{equation*}
    p-\alpha
    =\frac{\alpha\tau^2}{\alpha\tau^2+\beta}-\alpha
    =\frac{\alpha\beta(\tau^2-1)}{\alpha\tau^2+\beta}>0,
\end{equation*}
since $\tau>1$. For the upper bound,
\begin{equation*}
    \beta-p
    =\beta-\frac{\alpha\tau^2}{\alpha\tau^2+\beta}
    =\frac{\beta^2-\alpha^2\tau^2}{\alpha\tau^2+\beta}
    =\frac{(\beta-\alpha\tau)(\beta+\alpha\tau)}{\alpha\tau^2+\beta}>0,
\end{equation*}
since the assumption $\tau<\beta/\alpha$ ensures $\beta-\alpha\tau>0$.

Since $p\in(\alpha,\beta)$ and $\alpha+\beta=1$, we have 
$p\in(\alpha,1-\alpha)$. The function $v(x)=x(1-x)$ is strictly 
concave and symmetric about $x=1/2$, so it is strictly increasing 
on $(\alpha,1/2]$ and achieves its minimum on the interval 
$[\alpha,1-\alpha]$ at the endpoints, where $v(\alpha)=v(1-\alpha)
=\alpha\beta$. Since $p$ lies strictly in the interior, we conclude
\begin{equation*}
    p(1-p)>\alpha\beta.
\end{equation*}

\emph{Part 3: Verifying $F(\alpha\beta)\ge\Gamma^2$.}
Since $F$ is strictly increasing and $p(1-p)>\alpha\beta$, we have 
$\sqrt{R}=F(p(1-p))>F(\alpha\beta)$. It now suffices to show 
$F(\alpha\beta)>\Gamma^2$. The different between them is
\begin{equation*}
    F(\alpha\beta)-\Gamma^2
    =\frac{\rho\left[2\alpha\beta(\rho-1)
    -(\alpha\rho-\beta)\right]
    \left[2\alpha\beta(\rho-1)
    +(\alpha\rho-\beta)\right]}
    {(\rho+(\rho-1)^2\alpha\beta)(\alpha\rho+\beta)^2}.
\end{equation*}
The denominator is positive. For the value of $2\alpha\beta(\rho-1)-(\alpha\rho-\beta)$, substituting $\beta=1-\alpha$ and expanding gives
\begin{equation*}
    2\alpha\beta(\rho-1)-(\alpha\rho-\beta)
    =2\alpha(1-\alpha)(\rho-1)-\alpha\rho+(1-\alpha)
    =(1-2\alpha)(1+\alpha(\rho-1))>0,
\end{equation*}
since $\alpha<1/2$ and $\rho>1$. In addition, the value of $2\alpha\beta(\rho-1)+(\alpha\rho-\beta)$ is positive because 
both terms are positive (recall $\alpha\rho>\beta$). Therefore $F(\alpha\beta)>\Gamma^2$, 
and combining with Part~2 gives $\sqrt{R}>F(\alpha\beta)>\Gamma^2$. Thus, we conclude the proof of Proposition~\ref{prop:muon_contract}.

\end{document}